\documentclass[nohyperref]{article}

\usepackage{microtype}
\usepackage{graphicx,booktabs,array}
\usepackage{makecell}
\usepackage{subfigure}
\usepackage{booktabs} 
\usepackage{multirow}
\usepackage{hyperref}

\usepackage[accepted]{icml2022}

\usepackage{amsmath}
\usepackage{amssymb}
\usepackage{mathtools}
\usepackage{amsthm}
\usepackage{bbm}

\usepackage[capitalize,noabbrev]{cleveref}

\theoremstyle{plain}
\newtheorem{theorem}{Theorem}[section]
\newtheorem{proposition}[theorem]{Proposition}
\newtheorem{lemma}[theorem]{Lemma}
\newtheorem{corollary}[theorem]{Corollary}
\theoremstyle{definition}

\theoremstyle{remark}

\usepackage[textsize=tiny]{todonotes}

\newcommand{\Prob}[1]{{\mathbb{P}\left( #1 \right)}}
\newcommand{\EE}{\mathbb{E}}
\newcommand{\RR}{\mathbb{R}}
\newcommand{\calN}{\mathcal{N}}
\newcommand{\Pa}{\text{Pa}}
\newcommand{\Cat}{\text{Cat}}

\DeclareMathOperator{\trace}{trace}

\tikzset{
    style A/.style={minimum width=2cm, inner sep=4pt, outer sep=4pt},
    style B/.style={minimum width=2.5cm, inner sep=4pt, outer sep=4pt},
    style B/.style={minimum width=3cm, inner sep=4pt, outer sep=4pt},
}

\icmltitlerunning{Causal disentanglement of multimodal data}

\begin{document}

\onecolumn
\icmltitle{Causal disentanglement of multimodal data}

\icmlsetsymbol{equal}{*}

\begin{icmlauthorlist}
\icmlauthor{Elise Walker}{xxx}
\icmlauthor{Jonas A. Actor}{xxx}
\icmlauthor{Carianne Martinez}{yyy,comp}
\icmlauthor{Nathaniel Trask}{upenn}

\end{icmlauthorlist}

\icmlaffiliation{xxx}{Center for Computing Research, Sandia National Laboratories, Albuquerque, NM, USA}
\icmlaffiliation{yyy}{Applied Information Sciences Center, Sandia National Laboratories, Albuquerque, NM, USA}
\icmlaffiliation{comp}{School of Computing and Augmented Intelligence, Arizona State University, USA}
\icmlaffiliation{upenn}{School of Engineering and Applied Science, University of Pennsylvania, USA}

\icmlcorrespondingauthor{Nathaniel Trask}{natrask@sandia.gov}

\icmlkeywords{Causal models, multimodal machine learning, physics-informed machine learning, variational inference, variational autoencoders, fingerprinting, mixture of experts}

\printAffiliationsAndNotice{} 

\vskip 0.3in

\begin{abstract}
Causal representation learning algorithms discover lower-dimensional representations of data that admit a decipherable  interpretation of cause and effect;
as achieving such interpretable representations is challenging,
many causal learning algorithms utilize elements indicating prior information, such as (linear) structural causal models, interventional data, or weak supervision. 
Unfortunately, in exploratory causal representation learning, such elements and prior information may not be available or warranted. Alternatively, scientific datasets often have multiple modalities or physics-based constraints, and the use of such scientific, multimodal data has been shown to improve disentanglement in fully unsupervised settings. 
Consequently, we introduce a causal representation learning algorithm (causalPIMA) that can use multimodal data and known physics to discover important features with causal relationships. 
Our innovative algorithm utilizes a new differentiable parametrization to learn a
directed acyclic graph (DAG) together with a 
latent space of a variational autoencoder in an end-to-end differentiable framework via a single, tractable evidence lower bound loss function. We place a Gaussian mixture prior on the latent space and identify each of the mixtures with an outcome of the DAG nodes; this novel identification enables feature discovery with causal relationships.
Tested against a synthetic and scientific datasets, our results demonstrate the capability of learning an interpretable causal structure while simultaneously discovering key features 
in a fully unsupervised setting.

\end{abstract}


\section{Introduction}

To achieve autonomous scientific discovery, scientists are rapidly collecting large scientific datasets with a growing number of complex modalities.
Such large, multimodal scientific datasets extend beyond the limits of human cognition and thereby necessitate ML-driven methods to identify hidden, underlying factors in the data~\cite{boyce2019autonomous, auto_sci}.
The field of disentangled representation learning seeks to identify hidden features of data through an interpretable latent representation~\cite{rep_learning}. 
Variational autoencoders (VAE) frameworks are often used in representation learning to provide a meaningful, disentangled representation of data in a latent space~\cite{higgins2017betavae}. 
Physics-informed multimodal autoencoders (PIMA) have demonstrated the ability to detect features in multimodal datasets while incorporating known physics to aid in disentanglement~\cite{pima}.

One shortcoming of representation learning methods is that they typically do not consider any causal relationships. Representation learning has long been used to describe \textit{how} random variables relate to each other based on observable data, but does not address the \textit{why} behind random variable correlations. For example, many VAE frameworks assume that features are independent. Real-world data have natural correlative and causal relationships, however.  
To capture causal dependencies between random variables, recent works introduce causal inference into representation learning~\cite{causalVAE, seigal, causalGAN}.

Causal inference commonly identifies a directed acyclic graph (DAG) on a set of random variables representing features, where each node of the DAG is a random variable and directed edges between the nodes represent a causal relationship~\cite{causal_rep_learn}. Traditionally, learning a DAG reduced to expensive, combinatorial searches, e.g.~\cite{ges, linngam, pc}. Recent methods of learning DAGs, however, utilize a continuous, differentiable optimization scheme, which bypasses the otherwise laborious combinatorial search in the space of all DAGs~\cite{zheng2018dags}.

Many works are interested in learning a causal DAG on human-specified features from data, or, alternatively, learning a data distribution given a known DAG. We, however, are interested \textit{causal representation learning}, which means learning a causal DAG in concert with learning a lower-dimensional representation of data. Current frameworks for causal representation learning often rely on causal structural models, interventional data, or labels on data features. Such assumptions and interventions are deemed necessary for finding unique, or indentifiable, DAGs. We, however, are considering the exploratory setting where we are not concerned as much with unique, identifiable models so much as identifying plausible causal patterns within datasets where no prior additional information on the causal characteristics is available.
In lieu of additional causal information or assumptions, we follow~\cite{pima} and instead rely on multiple modalities or physics-based constraints of the data.
In this paper we present causalPIMA: a fully unsupervised causal representation learning framework capable of handling multiple modalities and physics-based constraints.
In particular, we adapt a new DAG-learning structure to the latent space of the PIMA framework. The result is a multimodal variational autoencoder with physics-based decoder capabilities such that clustering in the latent space follows a DAG structure that is learned simultaneously with the variational autoencoder embedding. The ability to handle multimodal data with physical constraints makes our algorithm unique from other causal representation learning algorithms. 

\subsection{Related Works}
The references detailed below give a non-comprehensive overview of the current work in representation causal inference, as well as references that informed our algorithmic development.

\paragraph{Continuously learning DAGs.}
A continuous optimization strategy for learning DAGs is first introduced in~\cite{zheng2018dags}, where the key component was developing new conditions for enforcing acyclic DAGs. Works such as~\cite{yu2020neur} further build off of this idea and introduce new continuous constraints for learning DAGs. Applications of continuous optimization of DAGs include~\cite{yu2021dags, yu2019gnn, causalVAE}. In contrast, our DAG parametrization is inspired by Hodge theory~\cite{jiang2011statistical,lim2020hodge} where we view edges as the flows between nodes.
Furthermore, our novel parametrization includes a temperature parameter that regularizes the edge indicator function in order to avoid local minimum while training.

\paragraph{Causal representation learning.}
Much of causal inference seeks to fit a DAG to data, or otherwise already assume a DAG and looks to fit the data to the DAG. Causal representation learning aims to learn a lower-dimensional representation of data with a causal interpretation. In~\cite{causalVAE}, the authors introduce a linear structural causal model into a VAE framework. Their framework enables counterfactual data generation and has some identifiability guarantees under set assumptions. Their approach requires weak supervision, however, in order to achieve disentanglement and identifiability.
Our method differs from that of~\cite{causalVAE} in that our algorithm can handle multiple modalities, known physics, and, most significantly, is completely unsupervised. A fully unsupervised framework is necessary for truly exploratory settings, such as autonomous scientific discovery. 

A fully linear causal representation learning approach is introduced in~\cite{seigal}, where unimodal data is factored into a linear causal model and a lower-dimensional representation. The primary objective of this work is to provide identifiability analysis in causal disentanglement. Indeed, this unimodal method, given interventional data, is guaranteed to be identifiable given a pure intervention on each random variable of the DAG. Our work differs from that of~\cite{seigal} in that we do not assume linearity or interventional data. Indeed our work finds causal relationships and multimodal, nonlinear representations in settings where interventional data is not available.

\paragraph{Latent representations of scientific datasets.}
Scientific datasets often are artisan, consist of various modalities, and obey physics constraints. 
Physics-informed multimodal autoencoders (PIMA), introduced in~\cite{pima}, use a VAE framework to learn a joint representation of multimodal data with optional physical constraints on the decoders. In particular, they show that additional modalities can improve classification and disentanglement.
Consequently, we chose to base our causal algorithm on the PIMA framework. The result is that our algorithm has a tractable, closed-form evidence lower bound loss function and can also handle incomplete multimodal data and incorporate simulators, reduced-order models, or other physics-based predictions.

\begin{figure}
    \centering
    \includegraphics[width=\textwidth]{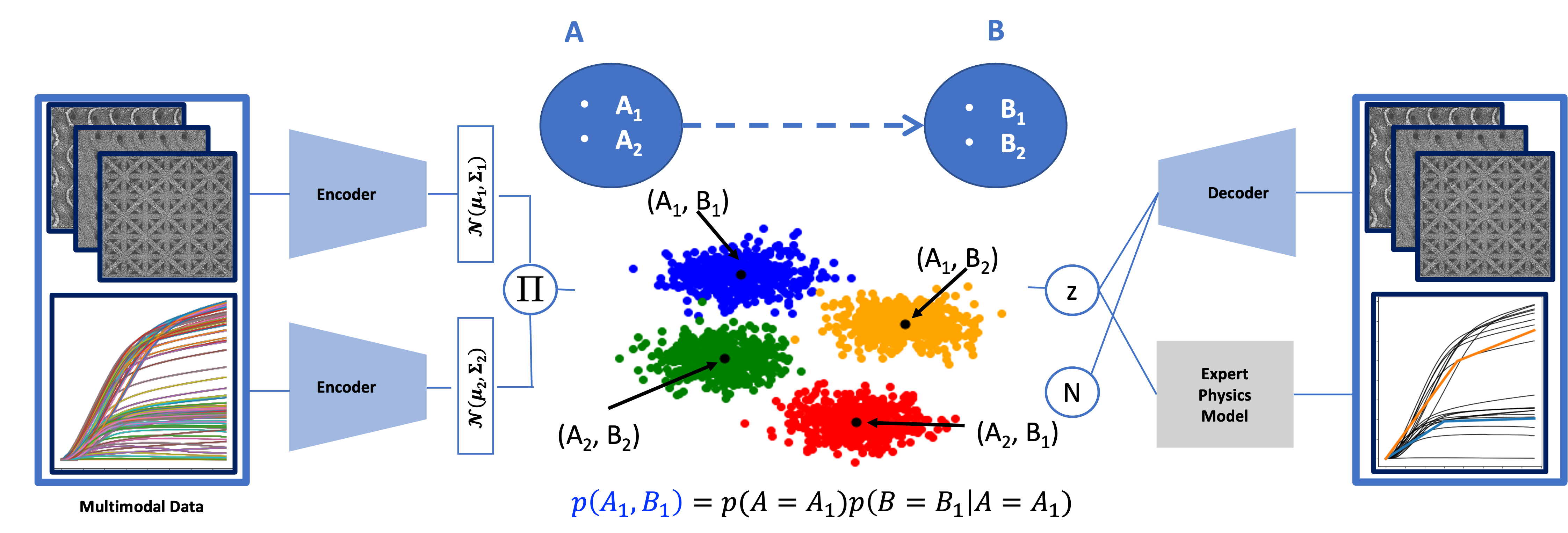}
    \caption{A cartoon description of the causalPIMA algorithm. Each modality is variationally encoded into latent space and unimodal embeddings are joined through a product of experts. Points in the latent space are clustered by a Gaussian mixture prior. Gaussian clusters are identified with outcomes of the nodes of a trainable DAG in that the probability of belonging to a given cluster is equal to the probability of an outcome of the DAG nodes. For example, the probability of belonging to the blue cluster is equal to the probability $p(A=A_1, B=B_1)$ where $A$ and $B$ are nodes in the trainable DAG. Points in the latent space are decoded back into each modality, with the option of incorporating known physics into the decoders. \label{fig:framework}}
\end{figure}

\section{Algorithm framework}
Given data $\mathbf{X} = \{ X_1, \dots, X_M \}$ from $M$ distinct modalities, we seek a common embedding into a latent space $Z \in \RR^J$, where the latent space representation admits distinct clusters based on encoded features of the data. 
We assume that our embedded representations are 
described by $L$ categorical features $\mathbf{N}_1,\dots ,\mathbf{N}_L$, where each feature $\mathbf{N_\ell}$ is a categorical random variable with $C_\ell$ outcomes. Moreover, we assume that there is a causal relationship between features, where the relationship is representable by a directed acyclic graph (DAG). That is, we assume that each feature $\mathbf{N}_\ell$ is a node in a DAG $G$. Letting $\text{Pa}(\mathbf{N}_\ell)$ denote the immediate parents of feature node $N_\ell$, we assume the joint distribution of the feature nodes follows the \textit{Markov factorization property}: 
\begin{equation}
p(\mathbf{N}) := p(\mathbf{N_1},\dots,\mathbf{N}_L) = \prod_{\ell=1}^L p(\mathbf{N}_\ell \vert \text{Pa}(\mathbf{N}_\ell)).
\end{equation}
We relate the features nodes of the DAG to the latent space $Z$ through identification. That is, we enforce a Gaussian mixture model (GMM) prior on the latent space $Z$. Each Gaussian in the mixture corresponds to a unique outcome of the joint distribution $\mathbf{N}$. In particular, the number of clusters in our GMM is $C = C_1 \cdots C_L$ and the categorical probability of each cluster is given by $\mathbf{N}$.
Appendix~\ref{app:nomen} contains a summary of notation (Table~\ref{table:notation})  as well as a sketch depicting how the causal graph of features is related to the latent space embedding (Figure \ref{fig:latent-space-sketch}).

We construct our embedding and latent space representation via a multimodal variational autoencoder, with distributions for the prior $p$ and posterior $q$. Following work such as~\cite{dilokthanakul2016deep, jiang2016variational}, we train our variational autoencoder through finding distributions $p$, $q$, and DAG $G$ which maximize the evidence lower bound (ELBO) loss:

\begin{equation}
\mathcal{L} = \mathbb{E}_{q(Z,N|\mathbf{X})} \left[ \log \frac{p(\mathbf{X},Z,\mathbf{N})}{q(Z,\mathbf{N}|\mathbf{X})} \right].
\end{equation}

We assume independence of decoding mechanisms for each modality for our prior, and assume mean-field separability for the posterior. These assumptions respectively give:

\begin{align} \label{eq:assumptions}
    p(\mathbf{X} \vert Z,\mathbf{N}) = \prod_{m=1}^M p(X_m \vert Z,\mathbf{N}) &\qquad\text{ and }\qquad
    q(Z,\mathbf{N} \vert \mathbf{X}) = q({Z\vert\mathbf{X}}) q({\mathbf{N}\vert\mathbf{X}}).
\end{align}

The ELBO above is computationally tractable through strategic framework decisions. In particular, 
our framework (1) utilizes unimodal deep encodings with Gaussian outputs, (2) fuses the unimodal deep encodings via a product of experts (PoE), (3) models clusters in the latent space as a mixture of Gaussians, (4) computes the probability of each cluster as the joint probability of the nodes a trainable DAG, and (5) utilizes a mixture of deep decoders with the optional capability of physics-informed decoders for modalities suitable to expert modeling. By assuming Equation~\ref{eq:assumptions} and extensively using Gaussians, the ELBO separates as sums of expectations of Gaussian distributions. In the case of Gaussians with diagonal covariance,~\cite{jiang2016variational} gives a closed-form solution to compute such an expectation (see Corollary~\ref{lemma:int-gaussian-log-gaussian} in Appendix~\ref{app:multivariate-gaussian-derivation}). For general Gaussian distributions, we give the closed-form solution in Lemma~\ref{lemma:int-multivariate-gaussian-log-gaussian} of Appendix~\ref{app:multivariate-gaussian-derivation}. For simplicity, we assume Gaussian distributions with diagonal covariance throughout this work.

Our algorithmic framework thus consists of (1) a multimodal variational autoencoder with a Gaussian mixture prior and (2) a parameterization of our DAG and the causal structure it induces. We describe each of these components in the subsections below.

\subsection{Multimodal variational autoencoder with Gaussian mixture prior}
Our mulitmodal representation learning framework amounts to a variational autoencoder (VAE) with a Gaussian mixture model (GMM) prior on the latent space $Z$. Our GMM is informed by a DAG $G$ where all nodes of $G$ are categorical random variables. In particular, we identify each Gaussian in the latent space with an outcome on the nodes of $G$. Thus the total number of Gaussians in the latent space is $C_1 \cdots C_L$, where $L$ is the number of nodes and $C_\ell$ is the number of outcomes of the $\ell^{th}$ node $\mathbf{N}_\ell$. The joint random variable of all nodes is $\mathbf{N}$, which indexes the clusters in the GMM. In essence, we are putting a causal prior on the distribution of the clusters so that the probability of belonging to cluster $(c_1, \dots, c_L)$ is given by $\mathbf{A}_{c_1, \dots, c_L}=p(\mathbf{N}_{c_1, \dots, c_L})$. We assume each cluster in the mixture is a Gaussian of the form $p(Z \vert \mathbf{N}_{c_1,\dots,c_L}) \sim \calN(\widetilde{\mu}_{c_1,\dots,c_L}, \widetilde{\sigma}_{c_1,\dots,c_L}^2 \mathbf{I})$, where the parameters $\widetilde{\mu}_{c_1,\dots,c_L}$ and $\widetilde{\sigma}_{c_1,\dots,c_L}^2$ are either trainable variables, or computed using block-coordinate maximization strategy outlined in Section~\ref{sec:training}. 

We handle the multimodal embedding and decoding of the VAE in the same manner as~\cite{pima}. In particular, we use neural network encoders to embed each modality as a Gaussian and then combine these embeddings using a PoE. That is, for each modality $m$ we assume $q(Z \vert X_m) \sim \calN(\mu_{m}, \sigma_{m}^2 \mathbf{I})$, where $[\mu_m, \sigma_m^2] = F_m(X_m; \theta_m)$ for a neural network $F_m$ with trainable parameters $\theta_m$. We deterministically compute the multimodal embedding from the unimodal ones via the identity $q(Z \vert \mathbf{X}) \sim \calN(\mu, \sigma^2 \mathbf{I}) = \alpha \prod_{m=1}^M \calN(\mu_{m}, \sigma_{m}^2\mathbf{I})$,  where $\alpha$ is a normalization constant and 
\begin{align} \begin{split}
    \sigma^{-2} = \sum_{m=1}^{M} \sigma_{m}^{-2} &\qquad\text{ and }\qquad \frac{\mu}{\sigma^2} = \sum_{m=1}^M \frac{\mu_{m}}{\sigma_{m}^2}.
\end{split}\end{align}
During training, the multimodal distribution is sampled using the reparametrization trick. That is, we sample $\epsilon \sim \calN(0, \mathbf{I})$ and compute $z = \mu + \epsilon \odot \sigma$, where $\odot$ is the Hadamard product.

Our decoders output a Gaussian for each modality $p(X_m \vert Z, \mathbf{N}_{c_1, \dots, c_L}) \sim \calN(\widehat{\mu}_{m;c_1, \dots, c_L}, \widehat{\sigma}_{m; c_1, \dots, c_L}^2 \mathbf{I})$. The Gaussians' parameters are determined by neural networks $D_{m; c_1, \dots, c_L}$, i.e. $[\widehat{\mu}_{m;c_1, \dots, c_L}, \widehat{\sigma}_{m; c_1, \dots, c_L}^2] = D_{m;c_1, \dots, c_L}(Z;\widehat{\theta}_{m;c_1, \dots, c_L})$. Alternatively our decoders $D_{m;c_1, \dots, c_L}(Z;\widehat{\theta}_{m;c_1, \dots, c_L})$ can be expert models, or they can depend upon only $Z$, i.e. $p(X_m \vert Z, \mathbf{N}_{c_1, \dots, c_L}) = p(X_m \vert Z)$.

\subsection{Directed acyclic graph and joint distribution of nodes} \label{sec:joint-def-A}
Given our data, we assume that hidden features - which are discovered by the encoder and decoder of our architecture - admit a causal structure. In particular, we assume a directed acyclic graph $G$ with $L$ nodes, where each feature node $\mathbf{N}_\ell$ is a categorical random variable representing a hidden feature with $C_\ell$ outcomes. Furthermore, the joint distribution of $\mathbf{N}$ factors according to $G$: 
\begin{equation} \label{eq:markov}
p({\mathbf{N}}) = p({\mathbf{N}_1,\dots,\mathbf{N}_L}) = \prod_{\ell=1}^L p({\mathbf{N}_\ell \vert \text{Pa}(\mathbf{N}_\ell)}).
\end{equation}
One challenge of causal representation learning is determining how to efficiently learn the DAG edges, which are described by $\text{Pa}(\mathbf{N}_\ell)$.
Building off of concepts from Hodge theory~\cite{jiang2011statistical, lim2020hodge}, we pose our regularized edge indicator function as the graph gradient ($\mathcal{G}$) on a set of nodes. By using the graph gradient, we are guaranteeing that our edge indicator function is curl-free, and consequently defines a complete DAG; we introduce sparsity in the complete DAG through nonnegative weightings ($B$) of edges.

Explicitly, given a set of nodes, each node $\mathbf{N}_\ell$ is assigned a trainable score $\xi_\ell$. We denote the vector of trainable node scores as $\vec{\xi}$. Each potential edge $e_{ij}$ between nodes is assigned a value $F_{ij}$ given by
\begin{equation} \begin{split}
F_{ij} = (B \cdot \mathcal{G}\xi)_{ij} = B_{ij}(\xi_j - \xi_i),
\end{split} \end{equation}
where $\mathcal{G}$ is the graph gradient operator, and $B$ is a trainable nonnegative metric diagonal tensor inducing sparsity in $G$.
We use these edge values to give a regularized edge indicator function
\begin{equation} \label{eq:e}
E_{ij} = \text{ReLU} \left( \tanh \left( \frac{1}{\beta}F_{ij} \right) \right),
\end{equation}
where the scalar $\beta > 0$ is a temperature parameter that controls the sharpness of the regularization of the indicator function. We use this temperature to control how easily the DAG can update during training, as described in Section~\ref{sec:practical}.
With this formulation, we assign edges in our DAG via the rule 
\begin{equation} \label{eq:dag-assignment}
 \mathbf{N}_i \subseteq \text{Pa}(\mathbf{N}_j) \qquad 
 \Longleftrightarrow \qquad \lim_{\beta \rightarrow 0} E_{ij} = 1 \qquad \Longleftrightarrow \qquad B_{ij} \ne 0 \text{ and } \xi_i < \xi_j.
\end{equation}

Our DAG parametrization does indeed guarantee a DAG and is flexible enough to learn any possible DAG. Formal proofs are given in Appendix~\ref{app:dag}.

\begin{theorem}
Let $A$ be the adjacency matrix of a directed graph $G$. Then $G$ is a DAG if and only if $A=\lim_{\beta \rightarrow 0} E$ for some matrix $E$ with entries given by Equation~\ref{eq:e}. 
\end{theorem}
\begin{proof}
See Lemma~\ref{lemma:isaDAG} and Proposition~\ref{prop:allDAGs} in Appendix~\ref{app:dag}.
\end{proof}

For a DAG $G$ parameterized with the edge scores in Equation ~\ref{eq:e}, we need to compute the joint probability distribution on the nodes $\mathbf{N}$, as given by the Markov factorization property (Equation ~\ref{eq:markov}). We now proceed to describe our representation of each of the terms in this factorization, i.e. for the probability distribution at each node $\vec{\pi}_\ell = p(\mathbf{N}_\ell \vert \text{Pa}(\mathbf{N}_\ell)).$ However, the direction of the edge dependencies in the DAG $G$ may change during training, and as a result the number of causal factors that are parents of a given node (i.e. $\text{Pa}(\mathbf{N}_\ell)$) may change as well. We therefore build, for each node $\ell$, a parameterization of these probabilities that allows for any subset of nodes to be parents via a trainable tensor $\mathbf{W}^\ell$; we downselect which nodes are parents using the DAG edge indicator scores from Equation ~\ref{eq:e} and averaging out those modes that are \emph{not} parents.
For ease of notation, we let $\mathbf{A} = p(\mathbf{N})$.

For each $\ell=1,\dots,L$, we define $\mathbf{W}^\ell$ to be a nonnegative rank-$L$ tensor of size $C_1 \times \dots \times C_L$, constrained so that for any $c_1,\dots,c_{\ell-1},c_{\ell+1},\dots,c_L$,
\begin{equation*} \sum_{c_\ell = 1}^{C_\ell} W^\ell_{{c_1},\dots,{c_L}} = 1. \end{equation*}
The entries of $\mathbf{W}^\ell$ represent the probabilities 
\begin{equation*} \begin{split}
\mathbf{W}^\ell :&= p( \mathbf{N}_\ell \vert \mathbf{N}_k \text{ for } k \ne \ell) \\
W^\ell_{{c_1},\dots,{c_L}} &= p(\mathbf{N}_\ell = \mathbf{1}_{c_\ell} \vert \mathbf{N}_k = \mathbf{1}_{c_k} \text{ for } k \ne \ell),
\end{split} \end{equation*}
where $\mathbf{1}_{c}$ is a one-hot encoding with the $c^{\text{th}}$ entry set to 1.

With this definition, we can now proceed to describe our downselection algorithm and parameterize the probabilities over the structure of a given DAG. 
If a node $\mathbf{N}_k$ is \emph{not} a parent node of $\mathbf{N}_\ell$, then we remove mode $k$ from $\mathbf{W}^\ell$ by contracting $\mathbf{W}^\ell$ against $\frac{1}{C_k}\mathbf{1}$ (where $\mathbf{1}$ is the vector of all ones) along mode $k$.
This contraction makes $\mathbf{N}_\ell$ independent of $\mathbf{N}_k$ when $\mathbf{N}_k$ is not a parent of $\mathbf{N}_\ell$. If the node $\mathbf{N}_k$ \emph{is} a parent of $\mathbf{N}_\ell$, then we can contract against the realizations of the parent one-hot encodings, since they are already known. We therefore end up with an expression for the categorical distribution on $\mathbf{N}_\ell \vert \text{Pa}(\mathbf{N}_\ell)$ via
\begin{equation} \begin{split}
\vec{\pi}_\ell = p(\mathbf{N}_\ell \vert \text{Pa}(\mathbf{N}_\ell)) &= \mathbf{W}^\ell \,\, \bar{\times}_1 v_1 \,\, \bar{\times}_2 v_2 \,\,\dots\,\, \bar{\times}_{\ell-1} v_{\ell-1} \,\, \bar{\times}_{\ell+1} v_{\ell+1} \,\,\dots \,\, \bar{\times}_L v_L \\
&= \sum_{\substack{k=1\\ k\ne \ell}}^L \sum_{c_k =1}^{C_k} W^\ell_{c_1,\dots,c_L} v_{k,c_k},\\
\end{split} \end{equation}
where $\bar{\times}_k$ denotes the contractive $n$-mode tensor product against mode $k$ (see \cite{kolda2009tensor, bader2006algorithm} for more details) and
\begin{equation}
v_k = \begin{cases}
\mathbf{N}_k & \text{ if } \mathbf{N}_k \subseteq \text{Pa}(\mathbf{N}_\ell) \\
\frac{1}{C_k}\mathbf{1} & \text{ if } \mathbf{N}_k \subseteq \text{Pa}(\mathbf{N}_\ell)^\mathbf{c} \text{ and } k \ne \ell \\
\end{cases}
\end{equation}
Using the DAG representation from the previous section, these cases can be written as
\begin{equation}
v_k = \begin{cases}
\mathbf{N}_k & \text{ if } E_{k,\ell} = 1 \\
\frac{1}{C_k}\mathbf{1} & \text{ if } E_{k,\ell} \ne 1 \text{ and } k \ne \ell \\
\end{cases}
\end{equation}
or more concisely as
\begin{equation}
v_k = \frac{1}{C_k}\mathbf{1} - E_{k,\ell}\left(\frac{1}{C_k}\mathbf{1}-\mathbf{N}_k\right) \text{ for }k\ne \ell,
\label{eq:DAG-prob-causal}
\end{equation}
which has the benefit of allowing us to handle relaxations of $E$ where $E$ is not necessarily a binary matrix (such as when the temperature $\beta$ is small but not yet sufficiently close to 0). With $v_k$ defined as such, we can write $\vec{\pi}_\ell$ via
\begin{equation} \label{eq:pi-ell}
\vec{\pi}_\ell = \sum_{\substack{k = 1 \\ k \ne \ell}}^L \sum_{c_k = 1}^{C_k} W^\ell_{c_1,\dots,c_L} \left( \frac{1}{C_k}\mathbf{1} - E_{k,\ell}\left(\frac{1}{C_k}\mathbf{1}-\mathbf{N}_k\right) \right)_{c_k}.
\end{equation}

From the tensors $\mathbf{W}^\ell$ and the vectors $\vec{\pi}_\ell$, we can now compute $\mathbf{A}$.

By Corollary \ref{cor:order-traversal}, assume that the categorical variables $\mathbf{N}_1,\dots,\mathbf{N}_L$ are ordered such that $\forall k < \ell$, $\mathbf{N}_k \subset \text{Anc}(\mathbf{N}_\ell)$; if not, we reassign the indices via the permutation $\sigma$. Observe that
\begin{equation} \begin{split}
\mathbf{A}^\ell &:= p(\mathbf{N}_1,\dots,\mathbf{N}_\ell) \\
&= p(\mathbf{N}_\ell \vert \mathbf{N}_1,\dots,\mathbf{N}_{\ell-1}) p( \mathbf{N}_1,\dots,\mathbf{N}_{\ell-1}) \\
&= p(\mathbf{N}_\ell \vert \text{Pa}(\mathbf{N}_\ell)) \, \mathbf{A}^{\ell-1} \\
\end{split} \end{equation}
where $\mathbf{A}^0 = 1$, and where the expression for $p(\mathbf{N}_\ell \vert \text{Pa}(\mathbf{N}_\ell))$ is given by $\vec{\pi}_\ell$ in Equation \eqref{eq:pi-ell}. This inductive process of computing $\mathbf{A} = \mathbf{A}^L$ is given in Algorithm \ref{alg:algorithm-A}. 
\begin{algorithm}
    \begin{algorithmic}
        \REQUIRE $E$ an upper-triangular DAG score matrix
        \REQUIRE $\mathbf{W}^1,\dots,\mathbf{W}^L$ tensors, where $\mathbf{W}^\ell = p(\mathbf{N}_\ell \vert \mathbf{N}_k \text{ for } k \ne \ell) $
        \STATE $\mathbf{A}^{0}$ = 1
        \FOR{$\ell=1$ to $L$}
            \STATE $\mathbf{W}^\ell \leftarrow \texttt{reduce\_mean(}\mathbf{W}^\ell\texttt{, axis=[$\ell+1,\dots,L$], keepdims=False)}$ 
            \FOR{$k=1$ to $\ell-1$  }
                \STATE $\mathbf{W}^\ell \leftarrow E_{\ell, k}\cdot \mathbf{W}^\ell + (1.0-E_{\ell, k})\cdot \texttt{reduce\_mean(}\mathbf{W}^\ell\texttt{, axis=k, keepdims=True)}$
            \ENDFOR
            \STATE $\mathbf{A}^{\ell} \leftarrow \mathbf{W}^\ell_{\dots,:} \odot \mathbf{A}^{\ell-1}$
        \ENDFOR
        \OUTPUT $\mathbf{A} \leftarrow \mathbf{A}^L$ 
    \end{algorithmic}
    \caption{Algorithm for computing joint distribution kernel $\mathbf{A}$     \label{alg:algorithm-A}}
\end{algorithm}

\section{ELBO loss and training}
\subsection{Single sample ELBO}
The ELBO loss for training is identical to that of~\cite{pima}, albeit with different notation and computation of cluster assignment. The full ELBO derivation is in Appendix~\ref{app:elbo}. After dropping constant terms, the single-sample ELBO is:
\begin{equation} \label{eq:abbrev_elbo} \begin{split}
\mathcal{L} &= - \sum_{m=1}^M \log (\widehat{\sigma}_m^2) + \left \lVert \frac{X_m - \widehat{\mu}_m}{\widehat{\sigma}_m}\right \rVert^2 \\
&\qquad + \sum_{j=1}^{J} \log(\sigma_{j}^2) \\
&\qquad + \sum_{c_1=1}^{C_1}\cdots \sum_{c_L=1}^{C_L} \gamma_{c_1,\dots,c_L} \cdot \left[ 2 \log\left(\frac{\mathbf{A}_{c_1,\dots,c_L}}{\gamma_{c_1,\dots,c_L}}\right) -  \sum_{j=1}^{J} \log (\widetilde{\sigma}_{c_1,\dots,c_L;j}^2) + \frac{\sigma_{j}^2}{\widetilde{\sigma}_{c_1,\dots,c_L;j}^2} + \frac{(\mu_{j} - \widetilde{\mu}_{c_1,\dots,c_L;j})^2}{\widetilde{\sigma}_{c_1,\dots,c_L;j}^2}  \right], \\
\end{split} \end{equation}
where $\gamma_{c_1,\dots,c_L}$ is an estimate for the posterior distribution and, following~\cite{jiang2016variational}, is computed by:
\begin{align} \label{eq:gamma-expr} \begin{split}
\gamma :&= q(\mathbf{N} \vert \mathbf{X}) = p(\mathbf{N} \vert Z) = \frac{ p(\mathbf{N}) p(Z \vert \mathbf{N}) }{p(Z)} \\
\gamma_{c_1,\dots,c_L} &= \frac{p(\mathbf{N}_{c_1,\dots,c_L}) p(Z \vert \mathbf{N}_{c_1,\dots,c_L})}{ \sum_{c'_1=1}^{C_1}\cdots \sum_{c'_L=1}^{C_L} p(\mathbf{N}_{c'_1,\dots,c'_L})  p(Z \vert \mathbf{N}_{c'_1,\dots,c'_L}) } \\
&= \frac{ \mathbf{A}_{c_1,\dots,c_L} p(Z \vert \mathbf{N}_{c_1,\dots,c_L}) }{ \sum_{c'_1=1}^{C_1} \cdots \sum_{c'_L=1}^{C_L} \mathbf{A}_{c'_1,\dots,c'_L} p(Z \vert \mathbf{N}_{c'_1\dots c'_L})} \\
\end{split}\end{align}
where we recall $\mathbf{A} := p(\mathbf{N})$ for convenience. Note that $\mathbf{A}$ and $\gamma$ are both tensors with $L$ modes, of size $C_1 \times \dots \times C_L$. The tensor $\mathbf{A}$ can be calculated via Algorithm \ref{alg:algorithm-A} in Section \ref{sec:joint-def-A}. The values of $p(Z \vert \mathbf{N})$ can be computed by sampling from each Gaussian in the Gaussian mixture model. All other values are parameters in our model; Table~\ref{table:distributions} in Appendix~\ref{app:nomen} summarizes the assumed distributions on each term in the architecture, and lists how the variables are computed and updated during training.

\subsection{Training} \label{sec:training}
To train our causal model, we seek to maximize the ELBO $\mathcal{L}$ over the entire dataset. That is, if we use $\mathcal{L}_d$ to denote Equation~\ref{eq:abbrev_elbo} for the $d^{th}$ datapoint, then we want to minimize $-\sum_d \mathcal{L}_d$. 
Throughout training we alternate between (1) updating the neural network, expert model, and DAG parameters via gradient descent and (2) updating the Gaussian mixture centers and variances using block-coordinate maximization, similar to~\cite{pima}. In particular we compute the optimal Gaussian mixture centers and variances by taking the derivative of $-\sum_d\mathcal{L}_d$ with respect to the cluster centers and variances and solving for the global minimizers: 
\begin{align} \label{eq:gmm} \begin{split}
\widetilde{\mu}_{c_1,\dots,c_L}  &= \frac{\sum_d  \mu^{(d)}  \gamma_{c_1, \dots, c_L}^{(d)}}{\sum_d  \gamma_{c_1, \dots, c_L}^{(d)}},   \\
\widetilde{\sigma}_{c_1,\dots,c_L}^2 &=  \frac{\sum_d  (( \mu^{(d)} - \widetilde{\mu}_{c_1,\dots,c_L} )^2 + \sigma^{2(d)})  \gamma_{c_1, \dots, c_L}^{(d)}}{\sum_d  \gamma_{c_1, \dots, c_L}^{(d)}},  
\end{split}
\end{align}
where $d$ indexes the $ d^{th}$ data point, and, in particular $\mu^{(d)}$ and $\sigma^{2(d)}$ are respectively the encoded mean and variance of the $d^{th}$ data point. Our training procedure follows Algorithm~\ref{alg:training}.

\begin{algorithm}
    \begin{algorithmic}
        \INPUT data $\mathbf{x}=\{X_1, \dots, X_M\}$ in batches $\mathcal{B}$
        \STATE Compute $\gamma_{c_1,\dots, c_L}^{(d)}$ for all $x^{(d)} \in \mathbf{x}$ via Equation~\ref{eq:gamma-expr}
        \STATE Compute $\widetilde{\mu}_{c_1, \dots, c_L}$ and $\widetilde{\sigma}_{c_1, \dots, c_L}^2$ via Equation~\ref{eq:gmm}
        \FOR{$i=1$ to $N_{epochs}$}
            \FOR{$b \in \mathcal{B}$  }
                \STATE Perform optimizer update on ELBO
            \ENDFOR
            \STATE Calculate $\gamma_{c_1, \dots, c_L}$ for all $x^{(d)} \in \mathbf{x}$ via Equation~\ref{eq:gamma-expr}
            \STATE Update $\widetilde{\mu}_{c_1, \dots, c_L}$ and $\widetilde{\sigma}_{c_1, \dots, c_L}^2$ via Equation~\ref{eq:gmm}
        \ENDFOR
    \end{algorithmic}
    \caption{Training algorithm for causalPIMA \label{alg:training}}
\end{algorithm}

\subsection{Practical considerations}\label{sec:practical}
We implement several tools for algorithm adaptation and to aid in training. These tools are described here, and the use of these tools in each experiment is detailed in Appendix~\ref{app:arch}.

\paragraph{Pre-training.} 
Before fitting a DAG, we need a reasonable latent embedding. Thus we implemented a pre-training regimen following~\cite{jiang2016variational, Kingma2014auto}. 
As our first step in pre-training, we fix the pre-initialized encoders and initialize the cluster means and variances via Equations~\ref{eq:gamma-expr} and~\ref{eq:gmm}. This initialization has the benefits of providing a good GMM fit for the initial latent embedding, but if the initial latent embedding is poor or undiscriminating, then the initial GMM fitting by these step might not focus on any informative features. Our next step in pre-training is to train the weights and biases of the encoders and decoders via the reconstruction term or by fitting a unit-normal Gaussian variational autoencoder~\cite{Kingma2014auto}. Following this training we find a good initial GMM fit through iterations of Equations~\ref{eq:gamma-expr} and~\ref{eq:gmm}. This has the benefit of finding a good initial embedding from which the block coordinate maximization can recover meaningful features.

\paragraph{Edge indicator function adaptations.} 
We have two optional adaptations to the edge indicator function. The first is to add random noise to the node scores $\mathbf{\xi}$. This noise is included to break free of local minima, and may additionally test edge orientation. Our second adaptation is to anneal $\beta$ during training. In Equation~\ref{eq:e}, $\beta$ serves as a temperature parameter and, as $\beta \to 0$, $E$ approaches a true indicator function. Our annealing implementation is simple, where we specify the initial $\beta$, the final $\beta$, and the update frequency of $\beta$.

\paragraph{Updates on GMM parameters.}
The cluster center and variance updates in Equation~\ref{eq:gmm}, paired with the gamma calculation in Equation~\ref{eq:gamma-expr}, are reminiscent of expectation-maximization. The traditional maximization step would, however, also update the probability of belonging to each cluster ($\mathbf{A}$). While there is not a closed-form expression for an update on $\mathbf{A}$ from our ELBO since they depend on the underlying causal factorization, we alternatively perform extra gradient-descent steps to update $\mathbf{A}$ after each update of the cluster means and variances. Furthermore, we also implemented the option to perform multiple iterations of GMM variable updates per epoch.

\section{Experiments}\label{sec:experiments}
We tested causalPIMA on a synthethic dataset consisting of circle images and a materials dataset consisting of 3D printed lattices.
All architectures and hyperparameters for the experiments can be found in Appendix~\ref{app:arch}.

\begin{figure}[h!]
    \centering
    \includegraphics[width=\textwidth]{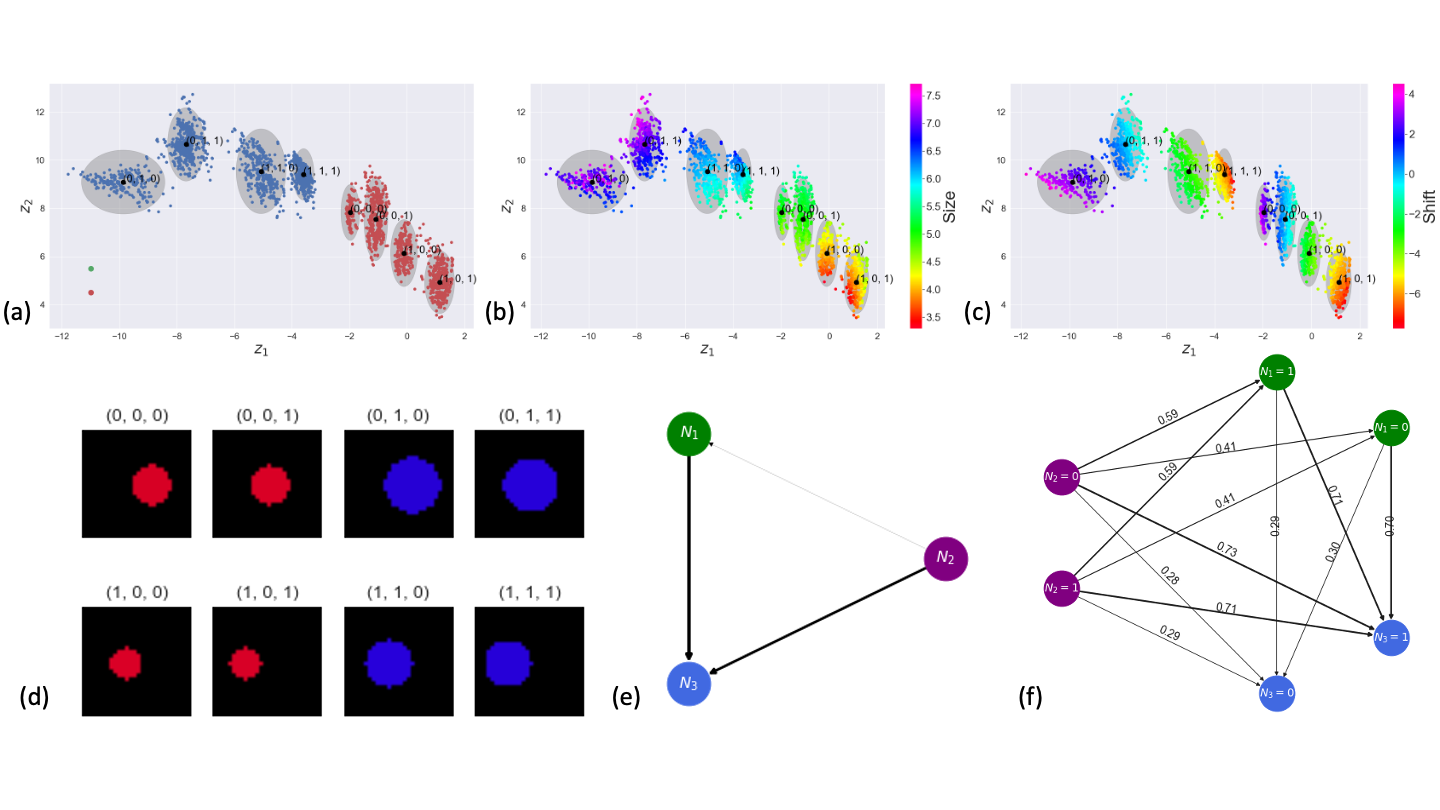}
    \caption{Results from the synthetic circles dataset. Subpanels (a), (b), and (c) show the latent space colored by hue, radius, and shift, respectively. Panel (d) shows the image of the datapoint nearest the cluster mean for each cluster. Subpanel (e) contains the learned DAG where edges are weighted by the probability of each edge. Subpanel (f) gives the probability of feature node $\mathbf{N}_\ell = n$ given parent node $\mathbf{N}_j=m$. By identifying key features in each cluster with its label, we see that $\mathbf{N}_1$ represents circle radius, $\mathbf{N}_2$ represents hue, and $\mathbf{N}_3$ represents shift.     \label{fig:circle_results} }
\end{figure}

\subsection{Circles}\label{ex:circles}

For our first experiment, we generated a synthetic dataset consisting of images of circles with three different features: hue $h$ (red, blue), radius $r$ (Gaussian mixture of big, small), and shift $s$ (Gaussian mixture of left, right). We generated 4096 circles using the decision tree in Figure~\ref{fig:circle_tree} where we purposefully overlapped distributions of $h$, $r$, and $s$ to necessitate the discovery of a DAG describing the generative process. We ran this experiment with three nodes in the DAG, where each node was a binary categorical random variable. The latent space showed disentanglement in the three different features. The learned DAG is in subpanel (e) of Figure~\ref{fig:circle_results}. By comparing cluster labels to features characteristic of each cluster, we see that node $\mathbf{N}_1$ in the DAG corresponds to radius, node $\mathbf{N}_2$  corresponds to hue, and node $\mathbf{N}_3$ corresponds to shift. For example, all clusters with red circles have a 0 in the second entry of their label. Under this identification, the resulting directed acyclic graph demonstrates that radius and hue play a key role in the outcome of shift.

\begin{figure}[h!]
    \centering
    \includegraphics[width=\textwidth]{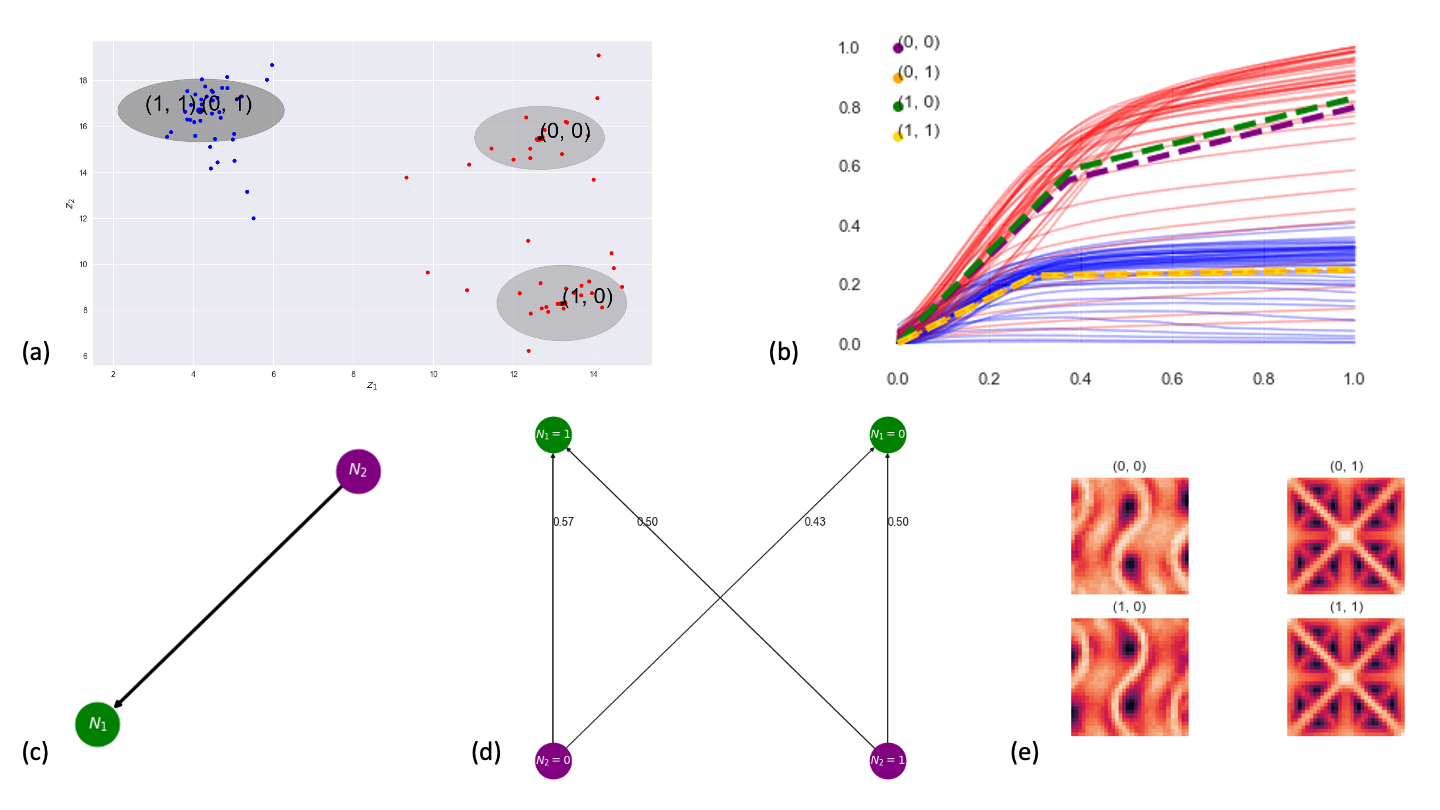}
    \caption{Results from the lattice experiment. Subpanel (a) shows the latent space, where two standard deviations for each cluster are shown as gray ellipses. Points in the latent space are colored by lattice type: gyroid (red), and octect (blue). Subpanel (b) shows the stress-strain curves colored by lattice type (gyroid in red, octect in blue) as well as the expert model for each cluster shown by dashed lines. Subpanel (c) shows the learned DAG, subpanel (d) shows gives the probability of feature node $\mathbf{N}_\ell = n$ given parent node $\mathbf{N}_j=m$. Subpanel (e) shows the mean of each cluster decoded as an image. By comparing the cluster labels with the nodes of the graph, we see that $\mathbf{N}_2$ corresponds to lattice type and influences $\mathbf{N}_1$, which corresponds to the stress-strain curve. This suggests that the mechanical response (stress-strain curve) follows from the microstructure imagery.     \label{fig:lattice_results}}
\end{figure}

\subsection{Lattices} \label{ex:lattices}
Our next experiment uses a dataset of 3D printed lattices~\cite{garland2020deep}. Two different lattice geometries were printed (octet and gyroid) with a total of 91 samples. An image ($X_1$) and a stress-strain curve ($X_2$) produced by a high-thoughput uniaxial compression machine were collected for each printed lattice. The stress/strain curves represent a physics-imbued modality where curves can be modeled via a continuous piecewise linear function.
Consequently, for the stress-strain modality, we used an expert model decoder composed of a two piecewise linear segments. Specifying two binary feature nodes resulted in a latent space organized by lattice type and stress-strain curves. The two clusters consisting of the octet geometry merged, and the corresponding expert models are nearly identical. This result is consistent with distribution of octet stress-strain curves, which has a lower variance than the gyroid stress-strain curves. By comparing cluster labels to features characteristic of each cluster, we see that node $\mathbf{N}_2$ corresponds to lattice type while $\mathbf{N}_1$ corresponds to the stress-strain curve profile.
The learned DAG suggests that the lattice type influences the stress-strain curve.

\section{Conclusion}
Causal disentanglement often relies upon interventional data and underlying model assumptions. For exploratory cases where such information is not available, we introduce a causal disentanglement algorithm that does not make any structural assumptions and does not rely on interventional data. Furthermore, this algorithm is capable of handling multiple modalities and underlying physics to encourage data-driven disentanglement of data with a causal interpretation. We demonstrate the efficacy of our algorithm on synthetic and real data and were able to achieve interpretable causal relationships. These results show that meaningful causal disentanglement is possible, even in purely exploratory settings. Future work will include methods for optionally introducing interventions and structural causal models.

\section*{Acknowledgements}
This article has been co-authored by employees of National Technology \& Engineering Solutions of Sandia, LLC under Contract No. DE-NA0003525 with the U.S. Department of Energy (DOE). The employees are solely responsible for its contents. Any subjective views or opinions that might be expressed in the paper do not necessarily represent the views of the U.S. Department of Energy or the United States Government. SAND number: SAND2023-11515O

\bibliography{first_draft}
\bibliographystyle{icml2022}

\appendix

\section{Nomenclature and Representations} \label{app:nomen}
We include tables outlining our notation choices (Table~\ref{table:notation}) as well as the various distributions appearing in our algorithm (Table~\ref{table:distributions}). Furthermore we include Figure~\ref{fig:latent-space-sketch} to illustrate the connection between the trained DAG and the clusters in the latent space.

\begin{table}[htbp!]
\centering
\begin{tabular}{lll}
\hline
Notation & Meaning & Range \\
\hline
$\mathbf{X}$ & all modalities & \\
$X_m$ & data from $m^\text{th}$ modality & $m=1,\dots,M$  \\
\hline
$\mathbf{N}$ & all features &\\
$\mathbf{N}_\ell$ & one-hot vector of for the $\ell^\text{th}$ feature space & $\ell=1,\dots,L$ \\
$\mathbf{N}_{c_1,\dots,c_L}$ & one combination of features & $c_\ell = 1,\dots,C_L$\\
\hline
Z & latent space representation &  \\
\hline
\end{tabular}
\caption{List of notation for causalPIMA derivation. \label{table:notation}}
\end{table}

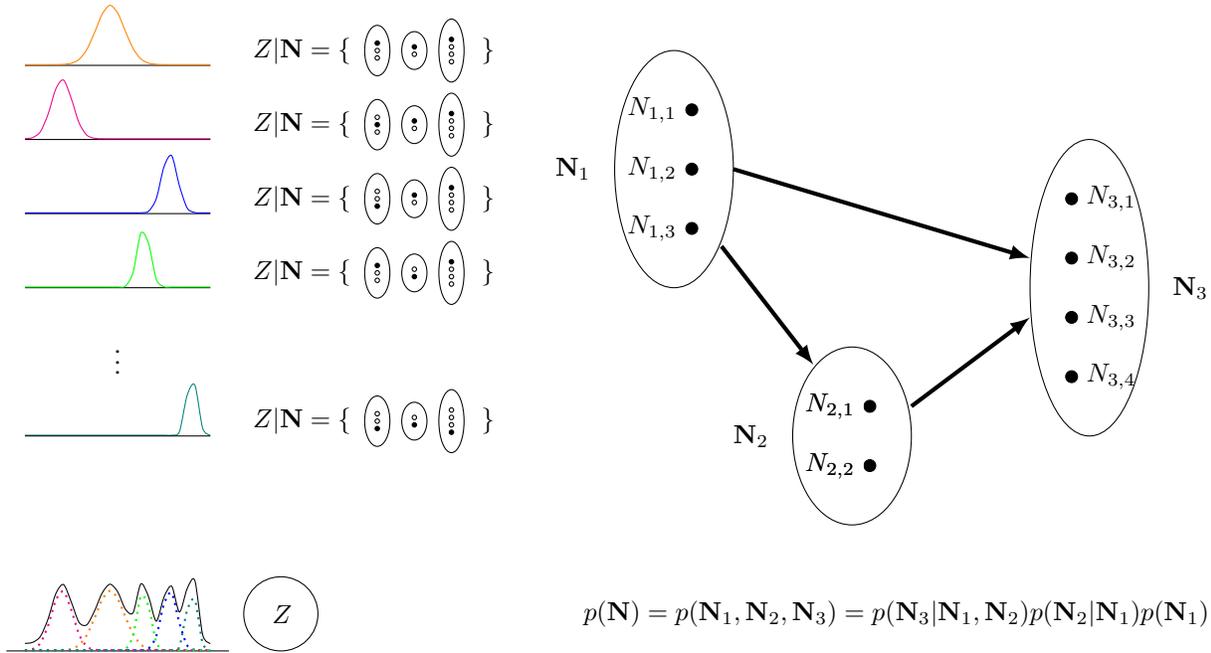
\begin{figure}[htbp!]
	\begin{center}
	\resizebox{0.95\textwidth}{!}{
	\begin{tikzpicture}

    \begin{scope}[scale=0.8, shift={(22.5,0)}]

           \foreach \i [count=\l from 1] in {1,...,3} \filldraw (-4.7,2-\l) circle (0.1);
           \foreach \i [count=\l from 1] in {1,...,2} \filldraw (-1.7,-3-\l) circle (0.1);
           \foreach \i [count=\l from 1] in {1,...,4} \filldraw (1.7,0.5-\l) circle (0.1);
           \foreach \i [count=\l from 1] in {1,...,3} \node[left] at (-4.8,2-\l) {${N_{1,\l}}$};
           \foreach \i [count=\l from 1] in {1,...,2} \node[left, fill=white, fill opacity=0.8] at (-1.8,-3-\l) {${N_{2,\l}}$};
           \foreach \i [count=\l from 1] in {1,...,2} \node[left] at (-1.8,-3-\l) {${N_{2,\l}}$};
           \foreach \i [count=\l from 1] in {1,...,4} \node[right] at (1.8,0.5-\l) {${N_{3,\l}}$};

           \draw (-5,0) ellipse (1cm and 2cm);
           \draw (-2,-4.5) ellipse (1cm and 1.5cm);
           \draw (2,-2) ellipse (1cm and 2.5cm);

           \draw[-latex, ultra thick] (-4.2,-1.3) -- (-2.65,-3.3);
           \draw[-latex, ultra thick] (-4,0) -- (1,-1.5);
           \draw[-latex, ultra thick] (-1,-4) -- (1,-2.5);

           \node[left=0.2cm] at (-6,0) {$\mathbf{N}_1$};
           \node[left=0.2cm] at (-3,-4.5) {$\mathbf{N}_2$};
           \node[right=0.2cm] at (3,-2) {$\mathbf{N}_3$};

    \end{scope}

           \foreach \i/\color [count=\l from 1] in {1/black,2/black,3/black,4/black,5/white,6/black} \draw[color=\color] (10,2.6-\l) ellipse (0.17cm and 0.34cm);
           \foreach \i/\color [count=\l from 1] in {1/black,2/black,3/black,4/black,5/white,6/black} \draw[color=\color] (10.5,2.6-\l) ellipse (0.17cm and 0.255cm);
           \foreach \i/\color [count=\l from 1] in {1/black,2/black,3/black,4/black,5/white,6/black} \draw[color=\color] (11,2.6-\l) ellipse (0.17cm and 0.425cm);

           \foreach \i/\color/\fc [count=\l from 1] in {1/black/black,2/black/white,3/black/white,4/black/black,5/white/white,6/black/white} \draw[color=\color,fill=\fc] (10,2.7-\l) circle (0.03cm);
           \foreach \i/\color/\fc [count=\l from 1] in {1/black/white,2/black/black,3/black/white,4/black/white,5/white/white,6/black/white} \draw[color=\color,fill=\fc] (10,2.6-\l) circle (0.03cm);
           \foreach \i/\color/\fc [count=\l from 1] in {1/black/white,2/black/white,3/black/black,4/black/white,5/white/white,6/black/black} \draw[color=\color,fill=\fc] (10,2.5-\l) circle (0.03cm);

           \foreach \i/\color/\fc [count=\l from 1] in {1/black/black,2/black/black,3/black/black,4/black/white,5/white/white,6/black/white} \draw[color=\color,fill=\fc] (10.5,2.65-\l) circle (0.03cm);
           \foreach \i/\color/\fc [count=\l from 1] in {1/black/white,2/black/white,3/black/white,4/black/black,5/white/white,6/black/black} \draw[color=\color,fill=\fc] (10.5,2.55-\l) circle (0.03cm);

           \foreach \i/\color/\fc [count=\l from 1] in {1/black/black,2/black/black,3/black/black,4/black/black,5/white/white,6/black/white} \draw[color=\color,fill=\fc] (11,2.75-\l) circle (0.03cm);
           \foreach \i/\color/\fc [count=\l from 1] in {1/black/white,2/black/white,3/black/white,4/black/white,5/white/white,6/black/white} \draw[color=\color,fill=\fc] (11,2.65-\l) circle (0.03cm);
           \foreach \i/\color/\fc [count=\l from 1] in {1/black/white,2/black/white,3/black/white,4/black/white,5/white/white,6/black/white} \draw[color=\color,fill=\fc] (11,2.55-\l) circle (0.03cm);
           \foreach \i/\color/\fc [count=\l from 1] in {1/black/white,2/black/white,3/black/white,4/black/white,5/white/white,6/black/black} \draw[color=\color,fill=\fc] (11,2.45-\l) circle (0.03cm);

           \node[right] at (8.2,1.6) {$Z \vert \mathbf{N} = \{ \qquad\qquad\quad \}$};
           \node[right] at (8.2,0.6) {$Z \vert \mathbf{N} = \{ \qquad\qquad\quad \}$};
           \node[right] at (8.2,-0.4) {$Z \vert \mathbf{N} = \{ \qquad\qquad\quad \}$};
           \node[right] at (8.2,-1.4) {$Z \vert \mathbf{N} = \{ \qquad\qquad\quad \}$};
           \node[right] at (8.2,-3.4) {$Z \vert \mathbf{N} = \{ \qquad\qquad\quad \}$};
            
           \foreach \i [count=\l from 1] in {1,...,4} \draw (5.25,2.4-\l)--(7.75,2.4-\l) ;
           \foreach \mu/\sigma/\color [count=\l from 1] in {6.4/0.27/orange, 5.75/0.2/magenta, 7.2/0.15/blue, 6.85/0.12/green, 6.6/0.1/white, 7.5/0.10/teal} \draw[color=\color,domain=5.25:7.75,smooth,variable=\t] plot ({\t},{0.8*exp(-1.0*((\t-\mu)/\sigma)^2) + 2.41-\l});
           \node at (6.5, -2.5) {$\vdots$};
           \draw (5.25,-3.6)--(7.75,-3.6);

           \draw (5,-6.5)--(8,-6.5);
           \foreach \mu/\sigma/\color [count=\l from 1] in {6.4/0.27/orange, 5.75/0.2/magenta, 7.2/0.15/blue, 6.85/0.12/green, 7.5/0.1/teal} \draw[dotted,thick,color=\color,domain=5.25:7.75,smooth,variable=\t] plot ({\t},{0.8*exp(-1.0*((\t-\mu)/\sigma)^2) -6.49});
           \draw[domain=5.25:7.75,smooth,variable=\t] plot ({\t},{0.8*(exp(-1.0*((\t-6.4)/0.27)^2)+exp(-1.0*((\t-5.75)/0.2)^2)+exp(-1.0*((\t-7.2)/0.15)^2)+exp(-1.0*((\t-6.85)/0.12)^2))+exp(-1.0*((\t-7.5)/0.10)^2)) - 6.4});

           \draw (8.7,-6) circle (0.5) ;
           \node[right] at (8.4,-6) {$\,Z$};
           \node at (17,-6) {$p(\mathbf{N}) = p(\mathbf{N}_1,\mathbf{N}_2,\mathbf{N}_3) = p(\mathbf{N}_{3} \vert \mathbf{N}_1, \mathbf{N}_2) p(\mathbf{N}_2 \vert \mathbf{N}_1) p(\mathbf{N}_1)$};

    \end{tikzpicture}
    }
 \end{center}
 \caption{Sketch of latent space and accompanying causal feature map. The latent space $Z$ is a Gaussian mixture model, where each individual Gaussian corresponds to a phenotype of features $\mathbf{N} = \{\mathbf{N}_1,\mathbf{N}_2,\mathbf{N}_3\}$ that obey the depicted causal relationship. \label{fig:latent-space-sketch}}
\end{figure}
\begin{table}[htbp!]
\centering
\begin{tabular}{llrll}
\hline
Distribution & Priors && Computation & Update \\
\hline
$p(X_m \vert Z, \mathbf{N})$ & $\mathcal{N}(\widehat{\mu}_m, \widehat{\sigma}_m^2 \mathbf{I}) $ & $[\widehat{\mu}_m, \widehat{\sigma}_m] $&$= D_m(Z; \widehat{\theta}_m)$ & trained $\widehat{\theta}_m$  \\[0.5em]
  \multirow{2}{*}{$p(Z \vert \mathbf{N}_{c_1,\dots,c_L})$} 
& \multirow{2}{*}{$\mathcal{N}(\widetilde{\mu}_{c_1,\dots,c_L}, \widetilde{\sigma}^2_{c_1,\dots,c_L}\mathbf{I})$} 
& $\displaystyle \widetilde{\mu}_{c_1,\dots,c_L} $&$= \frac{\sum_d \mu^{(d)} \gamma_{c_1,\dots,c_L}^{(d)}}{\sum_d \gamma_{c_1,\dots,c_L}^{(d)}}$ & computed \\[0.5em]  
& & $\displaystyle \widetilde{\sigma}_{c_1,\dots,c_L}^2 $&$= \frac{\sum_d((\mu^{(d)}-\widetilde{\mu}_{c_1,\dots,c_L})^2+\sigma^{2(d)})\gamma_{c_1,\dots,c_L}^{(d)}}{\sum_d \gamma_{c_1,\dots,c_L}^{(d)}}$ &  computed \\[0.5em]
$p(\mathbf{N}_\ell \vert \Pa(\mathbf{N}_\ell)) $ & $\Cat(\vec{\pi}_\ell) $ & $\vec{\pi}_\ell $&$= $ (Equation \eqref{eq:pi-ell}) & trained $\mathbf{W}^\ell, E$ \\[0.5em]
$q(Z \vert X_m)$ & $\mathcal{N}(\mu_{m}, \sigma^2_{m} \mathbf{I})$ & $[\mu_{m}, \sigma_{m}] $&$= F_{m}(X_m ; \theta_m)$ & trained $\theta_m$   \\[0.5em]
\multirow{2}{*}{$q(Z \vert \mathbf{X})$} & \multirow{2}{*}{$\mathcal{N}(\mu, \sigma^2 \mathbf{I})$} & $ \displaystyle \sigma^{2} $&$= \left( \sum_{m=1}^M \sigma_{m}^{-2} \right)^{-1}$ & computed \\[0.5em]
& & $ \displaystyle \mu $&$= \sigma^2 \sum_{m=1}^M \frac{\mu_{m}}{\sigma^2_{m}}$ & computed \\
\hline
\end{tabular}
\caption{Choices of distributions. \label{table:distributions}}
\end{table}

\section{Architectures, hyperparameters, and implementation} \label{app:arch}

We include details on architectures and implementation for each experiment. Hyperparameters for each experiment are in Table~\ref{table:hyperparams}.

\begin{table}[ht]
    \centering        
    \begin{tabular}{|c|c|c|c|c|}
        \hline
        & learning rate & encoding dim size & DAG node sizes & pre-training\\
        \hline
        Experiment~\ref{ex:circles} (circles) & 1e-6 & 2 & [2,2,2]  &  yes \\ 
        \hline
        Experiment~\ref{ex:lattices} (lattices) & 1.25e-5 & 2 & [2,2] & yes\\ 
        \hline

    \end{tabular}
    \caption{Hyperparameters for each experiment. \label{table:hyperparams}}
\end{table}

\subsection{Experiment~\ref{ex:circles} (Circles)}
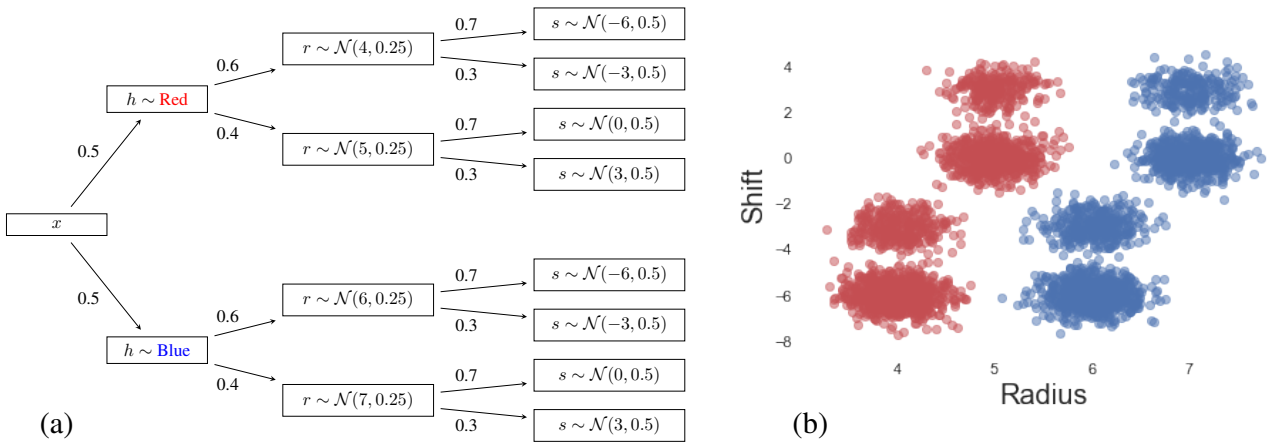
\begin{figure}[ht]
\centering
\resizebox{\textwidth}{!}{%
\begin{tikzpicture}[align=center]
\node [draw, style A] (x) at (1.0,0) {$x$};
\node [draw, style A] (r) at (3,2.5) {$h \sim $ \color{red}{Red}};
\node [draw, style A] (b) at (3,-2.5) {$h \sim $ \color{blue}{Blue}};
\node [draw, style B] (r1) at (7,3.5) {$r\sim\calN(4,0.25)$};
\node [draw, style B] (r2) at (7,1.5) {$r\sim\calN(5,0.25)$};
\node [draw, style B] (b2) at (7,-3.5) {$r\sim\calN(7,0.25)$};
\node [draw, style B] (b1) at (7,-1.5) {$r\sim\calN(6,0.25)$};
\node[draw, style B] (s1) at (12,4) {$s\sim\calN(-6,0.5)$};
\node[draw, style B] (s2) at (12,3) {$s\sim\calN(-3,0.5)$};
\node[draw, style B] (s3) at (12,2) {$s\sim\calN(0,0.5)$};
\node[draw, style B] (s4) at (12,1) {$s\sim\calN(3,0.5)$};
\node[draw, style B] (t1) at (12,-1) {$s\sim\calN(-6,0.5)$};
\node[draw, style B] (t2) at (12,-2) {$s\sim\calN(-3,0.5)$};
\node[draw, style B] (t3) at (12,-3) {$s\sim\calN(0,0.5)$};
\node[draw, style B] (t4) at (12,-4) {$s\sim\calN(3,0.5)$};

\draw[-stealth] (x) -- (r) node [midway, above left]{0.5};
\draw[-stealth] (x) -- (b) node [midway, below left]{0.5};
\draw[-stealth] (r) -- (r1) node [midway, above left]{0.6};
\draw[-stealth] (r) -- (r2) node [midway, below left]{0.4};
\draw[-stealth] (b) -- (b1) node [midway, above left]{0.6};
\draw[-stealth] (b) -- (b2) node [midway, below left]{0.4};
\draw[-stealth] (r1) -- (s1) node [midway, above left]{0.7};
\draw[-stealth] (r1) -- (s2) node [midway, below left]{0.3};
\draw[-stealth] (r2) -- (s3) node [midway, above left]{0.7};
\draw[-stealth] (r2) -- (s4) node [midway, below left]{0.3};
\draw[-stealth] (b1) -- (t1) node [midway, above left]{0.7};
\draw[-stealth] (b1) -- (t2) node [midway, below left]{0.3};
\draw[-stealth] (b2) -- (t3) node [midway, above left]{0.7};
\draw[-stealth] (b2) -- (t4) node [midway, below left]{0.3};
\node[inner sep=0pt] at (20,0)
    {\includegraphics[width=0.65\textwidth]{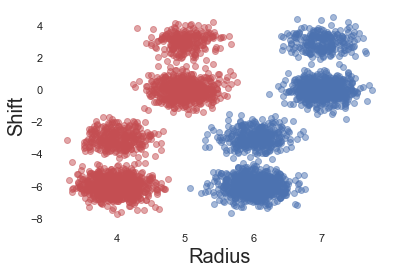}};
\node at (1,-4)  {\LARGE{(a)}};
\node at (16,-4) {\LARGE{(b)}}; 
\end{tikzpicture}
}
\caption{Generative information for synthetic circles dataset. Subpanel (a) contains the probability tree used to generate circle images. Edges of the tree are labeled by the probability of each feature given the previous feature(s). Subpanel (b) shows the distribution of the circle images, with hue represented by the color of each data point. \label{fig:circle_tree}}
\end{figure}

Our circles experiment consists of 4096 images of circles of size $28\times28\times3$. Our neural network architectures for this experiment were simple multilayer perceptrons (MLPs). The encoder for circles first flattens the image and then is an MLP consisting of five linear layers (with respective sizes 128, 64, 32, 16, and $2\times encoding\_dim$) with  ReLU activations between each layer. The final output has size $2\times encoding\_dim$ as it represents the the mean and standard deviation of the input in the latent space. The decoder for circles is also an MLP with five linear layers with ReLU activation between the layers. The respective linear layer sizes are 16, 32, 64, 128, and 2,352. The final layer is followed by a reshape into size $28\times 28\times 3$.

\subsection{Experiment~\ref{ex:lattices} (Lattices)}
The lattice dataset consists of 91 lattice samples where each sample contains an image and a stress-strain curve. Our data preparation follows the steps in~\cite{pima}, which we include here for completeness. In particular, the stress-strain curves were downsampled to an array of length 100 and normalized to have values in [0,1]. The lattice images were cropped and subsampled into images of size $32\times 32$ and standardized so each image had zero mean and unit variance over pixel intensity values. The dataset was further augmented by flipping images along each axis. We use an 81\%/9\%/10\% train/val/test split of the data.

Our neural network architectures for this experiment also follow those in~\cite{pima}, but we include the details here for completeness. We use relatively small convolutional encoders and decoders for the image modality. The image modality encoder consists of two 2D convolutional layers with 32 and 64 channels respectively, each with $3\times3$ kernels. We use the exponential linear unit (ELU) activation function as well as batch normalization after each convolutional layer, then pass the output to a fully connected layer of size $encoding\_dim \times 2$ to enable the representation of the mean and variances of each embedded point. The image decoder begins with a fully connected layer of appropriate size to be reshaped into 32 channels of 2D arrays, with each dimension having a length $\frac{1}{4}$ of the length of the number of pixels per side of the original image. We pass the reshaped output of the initial dense layer through a series of three deconvolution layers with 64, 32, and 1 channel, respectively, each with a kernel of size 3. The first two deconvolution layers use a stride of 3 and a ReLU activation function. The final deconvolution layer uses a stride of 1. No padding is used to retain the input shape while traversing these layers.

The stress-strain curve modality is treated as an expert modality as piecewise linear functions can capture the import aspects of a stress-strain curve. The encoder for the stress-strain curves is identical to the image encoder architecture, except we use 1D convolutions with 8 and 16 channels respectively in place of the 2D convolutional layers. The decoder is modeled as a continuous piecewise linear function consisting of two pieces. The trainable parameters for this decoder are inflection point and the slope of each linear piece.

\section{DAG Parameterization}
\label{app:dag}

We provide proofs showing that our construction guarantees a DAG, and that our parametrization can recover any DAG. Some of the proofs contain elements that are similar to those found in Appendix A of \cite{zheng2018dags}.

\begin{lemma}
Let $G = (\mathcal{V},\mathcal{E})$ be a graph with adjacency matrix $A$. Then, for any positive integer $k \ge 1$, $A^k_{ij}$ is the number of walks of length $k$ from $v_i$ to $v_j$.
\end{lemma}
\begin{proof}
We proceed by induction. The base case of $k=1$ is immediate, since $A$ is the adjacency of the matrix for $G$ and $G$ has no self-loops. Suppose the statement holds true for all walks of lengths up to and including length $k-1$. Then, the number of walks from $v_i$ to $v_j$ of length $k$ can be found by taking the number of walks of length $k-1$ from $v_i$ to an intermediate node $v_\ell$, and then completing one more step from $v_\ell$ to $v_j$, i.e.
$$\text{\# of walks} = \sum_{\ell = 1}^{\lvert \mathcal{V} \rvert} A^{k-1}_{i\ell} A_{\ell k} = A^k_{ij}.$$
\end{proof}
\begin{corollary} \label{cor:dag-trace-Ak}
A graph $G = (\mathcal{V}, \mathcal{E})$ with adjacency matrix $A$ has no cycles if and only if $\sum_{k=1}^\infty \text{trace}(A^k) = 0$.
\end{corollary}

\begin{lemma} \label{lemma:upper-triangular}
Let $\sigma = \texttt{argsort}(\xi)$ be any permutation that sorts $\xi$ in ascending order, and let $Q$ be the corresponding permutation matrix. Then, $Q E Q^T$ is strictly upper triangular.
\end{lemma}
\begin{proof}
Let $\text{Anc}(v)$ denote the \textit{ancestors} of a node $v \in \mathcal{V}$, defined as
\begin{equation}
\text{Anc}(v) = \{ w \in \mathcal{V} \backslash \{v\} \,\,:\,\, \text{ there exists a path from } w \text{ to } v \text{ in } G\}.
\end{equation}
We introduce the matrix $E^*$, given by
\begin{equation}
E^*_{ij} = \text{ReLU}\left( \tanh \left( \frac{1}{\beta} (\mathcal{G}\xi)_{ij} \right) \right).
\end{equation}
By our rules for DAG assignment in Equation \eqref{eq:dag-assignment},
\begin{equation} \mathbf{N}_i \subseteq \text{Anc}(\mathbf{N}_j) \qquad \Longleftrightarrow \qquad  \lim_{\beta \rightarrow 0} E^*_{ij} = 1 \qquad \Longleftrightarrow \qquad \xi_i < \xi_j.
\end{equation}
Let $\sigma = \texttt{argsort}(\xi)$. In the case $\sigma$ is not unique i.e. $\xi$ has repeated values, we break ties arbitrarily but consistently; therefore, without loss of generality, assume $\xi$ has no repeated values. Define $Q$ as the permutation matrix corresponding to $\sigma$. By definition of $Q$,
\begin{equation}
(Q\xi)_i < (Q\xi)_j \text{ for any } i \le j.
\end{equation}
We will show the following:
\begingroup
\renewcommand\labelenumi{(\theenumi)}
\begin{enumerate}
\item Fix $i,j$; for any $a < i$, we have $(Q E^* Q^T)_{ij} \le (Q E^* Q^T)_{aj}$, and if also $a < j$, the inequality is strict.
\item Fix $i,j$; for any $b > j$, we have $(Q E^* Q^T)_{ij} \le (Q E^* Q^T)_{ib}$, and if also $b > i$, the inequality is strict.
\item If both (1) and (2) are true, then both $Q E^* Q^T$ and $Q E Q^T$ are strictly upper triangular.
\end{enumerate}
\endgroup
For (1): Fix $j$. Define $\widehat{\xi} = \xi_{\sigma(j)} \mathbf{1} - \xi$. Then, for any $a < i$, $$ (Q \widehat{\xi})_i < (Q \widehat{\xi})_a.$$ 
Since $\tanh$ and $\text{ReLU}$ are monotonic nondecreasing,
\begin{equation} \begin{split}
(Q E^* Q^T)_{ij} 
&= \text{ReLU}\left( \tanh \left( \frac{1}{\beta} \left( (Q\xi)_j - (Q \xi)_i \right) \right) \right) \\
&= \text{ReLU} \left( \tanh \left( \frac{1}{\beta} (Q\widehat{\xi})_i \right) \right) \\
&\le \text{ReLU} \left( \tanh \left( \frac{1}{\beta} (Q\widehat{\xi})_a \right) \right) \\
&= \text{ReLU} \left( \tanh \left( \frac{1}{\beta} \left( (Q\xi)_j - (Q \xi)_a \right) \right) \right) \\
&= (Q E^* Q^T)_{aj}.
\end{split} \end{equation}
If additionally $a < j$, then $(Q\widehat{\xi})_a > 0$, and since $\tanh$ and $\text{ReLU}$ are strictly monotonic increasing on $(0,\infty)$, the inequality becomes strict.

For (2), repeat the proof of (1), with $\widehat{\xi} = \xi - \xi_{\sigma(i)}\mathbf{1}$.

For (3), since the range of $\text{ReLU}$ is nonnegative, for all $i,j$, $(QE^*Q^T)_{ij} \ge 0$. However, the diagonal entries $E^*_{ii} = 0$ for all $i$, so $(Q E^* Q^T)_{ii} = 0$ for all $i$ as well. By (1) and (2), for any $i,j$ below the diagonal, $E^*_{ij} \le 0$. Therefore, $QE^*Q^T$ is zero on or below the diagonal i.e. strictly upper triangular. Since the nonzero entries of $E$ are a subset of the nonzero entries of $E^*$, $Q E Q^T$ must be strictly upper triangular as well.
\end{proof}

\begin{lemma}\label{lemma:isaDAG}
Let $A = \lim_{\beta \rightarrow 0} E$ be the adjacency matrix of a directed graph $G = (\mathcal{V}, \mathcal{E})$. Then, $G$ is a DAG.
\end{lemma}
\begin{proof}
Let $\sigma$ be the permutation that sorts $\xi$ in ascending order and $Q$ be the corresponding permutation matrix.
Then, for any $\beta > 0$, by Lemma \ref{lemma:upper-triangular}, the matrix $Q E Q^T$ is strictly upper triangular, and therefore $QAQ^T = \lim_{\beta \rightarrow 0} Q E Q^T $ is strictly upper triangular. Additionally, since $Q$ is a permutation matrix, for any integer $k > 0$,
$$
\text{trace}((QAQ^T)^k) = \text{trace}(Q A^k Q^T) = \text{trace}(A^k),
$$
and since $QAQ^T$ is strictly upper triangular, the matrix $(QAQ^T)^k$ is strictly upper triangular as well.
Since the trace of any strictly upper triangular matrix is $0$, we conclude that $$\sum_{k=1}^\infty \text{trace}(A^k) = 0,$$ and therefore by Corollary \ref{cor:dag-trace-Ak},  $G$ is a DAG.
\end{proof}

\begin{corollary} \label{cor:order-traversal}
The permutation $\sigma$ provides the order to traverse the DAG in order.
\end{corollary}

\begin{proposition} \label{prop:allDAGs}
The edge parametrization in Equation~\ref{eq:e} is sufficiently expressive to represent all possible DAGs.
\end{proposition}
\begin{proof}
    The graph gradient is sufficient to recover any complete DAG. We recover any sub-DAG of any complete DAG by eliminating edges through the multiplication by the metric $B$.
\end{proof}

\section{ELBO derivation} \label{app:elbo}
We consider the ELBO loss
\begin{equation}
    \mathcal{L} = \EE_{q(Z,\mathbf{N} \vert \mathbf{X})} \left[\log \frac{p(\mathbf{X},Z,\mathbf{N})}{q(Z,\mathbf{N}\vert\mathbf{X})} \right].
\end{equation}

For convenience, we denote $\EE_{q(Z,\mathbf{N}\vert\mathbf{X})}$ as $\EE_q$. With our assumptions (Equation~\ref{eq:assumptions}), this ELBO expression becomes
\begin{equation} \begin{split}
\mathcal{L} &= \EE_{q} \left[\log \frac{p(\mathbf{X},Z,\mathbf{N})}{q(Z,\mathbf{N}\vert\mathbf{X})} \right] \\
&= \EE_q \log p(\mathbf{X},Z,\mathbf{N}) - \EE_q \log q(Z,\mathbf{N} \vert \mathbf{X}) \\
&= \EE_q \log \left( \left( \prod_{m=1}^M p(X_m \vert Z,\mathbf{N}) \right) p(Z\vert \mathbf{N}) p(\mathbf{N}) \right) - \EE_q \log \left( q(Z \vert \mathbf{X}) q(\mathbf{N} \vert \mathbf{X}) \right) \\
&= \sum_{m=1}^M \EE_q \log p(X_m \vert Z,\mathbf{N}) + \EE_q \log p(Z \vert \mathbf{N}) + \EE_q \log p(\mathbf{N}) - \EE_q \log q(Z\vert \mathbf{X}) - \EE_q \log q(\mathbf{N} \vert \mathbf{X}). \\
\end{split}  \label{eq:pima_elbo} \end{equation}

We estimate the distribution $q(\mathbf{N} \vert \mathbf{X})$ following \cite{jiang2016variational} by 
\begin{align} \label{eq:gamma-expr-derivation} \begin{split}
\gamma :&= q(\mathbf{N} \vert \mathbf{X}) = p(\mathbf{N} \vert Z) = \frac{ p(\mathbf{N}) p(Z \vert \mathbf{N}) }{p(Z)} \\
\gamma_{c_1,\dots,c_L} &= \frac{p(\mathbf{N}_{c_1,\dots,c_L}) p(Z \vert \mathbf{N}_{c_1,\dots,c_L})}{ \sum_{c'_1=1}^{C_1}\cdots \sum_{c'_L=1}^{C_L} p(\mathbf{N}_{c'_1,\dots,c'_L})  p(Z \vert \mathbf{N}_{c'_1,\dots,c'_L}) } \\
&= \frac{ \mathbf{A}_{c_1,\dots,c_L} p(Z \vert \mathbf{N}_{c_1,\dots,c_L}) }{ \sum_{c'_1=1}^{C_1} \cdots \sum_{c'_L=1}^{C_L} \mathbf{A}_{c'_1,\dots,c'_L} p(Z \vert \mathbf{N}_{c'_1\dots c'_L})} \\
\end{split}\end{align}
where we denote $\mathbf{A} := p(\mathbf{N})$ for convenience. Note that $\mathbf{A}$ and $\gamma$ are both tensors with $L$ modes, of size $C_1 \times \dots \times C_L$. The tensor $\mathbf{A}$ can be calculated via Algorithm \ref{alg:algorithm-A} in Section \ref{sec:joint-def-A}. The values of $p(Z \vert \mathbf{N})$ can be computed by sampling from each Gaussian in the Gaussian mixture model.

We now compute each expectation in Equation \eqref{eq:pima_elbo} using Corollary~\ref{lemma:int-gaussian-log-gaussian} to evaluate integrals.
\begin{enumerate}
    \item We compute $\EE_{q(Z,\mathbf{N} \vert \mathbf{X})}\log p(X_m \vert Z, \mathbf{N}) $ via the following:
    \begin{align} \begin{split}
        \EE_{q(Z,\mathbf{N}\vert\mathbf{X})} \log p(X_m \vert Z,\mathbf{N}) &= \log p(X_m \vert Z,\mathbf{N}) \\
        &= \log \left(\frac{1}{\sqrt{2\pi}\widehat{\sigma}_m}\right) - \frac{1}{2} \left \lVert \frac{X_m - \widehat{\mu}_m }{\widehat{\sigma}_m} \right \rVert^2\\
        &= -\frac{1}{2} \log (2 \pi \widehat{\sigma}_m^2) -  \frac{1}{2} \left \lVert \frac{X_m - \widehat{\mu}_m }{\widehat{\sigma}_m} \right \rVert^2.
    \end{split} \end{align}
    \item We compute $\EE_{q(Z, \mathbf{N} \vert \mathbf{X})} \log p(Z \vert \mathbf{N}) $ via the following:
    \begin{align}\begin{split}
        \EE_{q(Z, \mathbf{N} \vert \mathbf{X})} \log p(Z \vert \mathbf{N}) 
        &=  \sum_{\mathbf{N}} q(\mathbf{N} \vert \mathbf{X}) \int_Z q(Z \vert \mathbf{X}) \log p(Z \vert \mathbf{N}) dZ \\
        &= \sum_{c_1=1}^{C_1}\cdots \sum_{c_L=1}^{C_L}  q(\mathbf{N}_{c_1,\dots,c_L} \vert \mathbf{X}) \cdot \int_{Z} q(Z \vert \mathbf{X}) \log p(Z \vert \mathbf{N}_{c_1,\dots,c_L}) d Z \\
        &= \sum_{c_1=1}^{C_1}\cdots \sum_{c_L=1}^{C_L} \gamma_{c_1,\dots,c_L} \cdot \left[-\frac{1}{2} \sum_{j=1}^{J} \log 2\pi \widetilde{\sigma}_{c_1,\dots,c_L;j}^2 + \frac{\sigma_{j}^2}{\widetilde{\sigma}_{c_1,\dots,c_L;j}^2} + \frac{(\mu_{j} - \widetilde{\mu}_{c_1,\dots,c_L;j})^2}{\widetilde{\sigma}_{c_1,\dots,c_L;j}^2} \right],
    \end{split}
    \end{align}
    where $J = \dim Z$.   
    \item We compute $\EE_{q(Z, \mathbf{N} \vert \mathbf{X})} \log p{(\mathbf{N})}$ via the following:
    \begin{align}\begin{split}
        \EE_{q(Z, \mathbf{N} \vert \mathbf{X})} \log p{(\mathbf{N})} &= \sum_{\mathbf{N}} q(\mathbf{N} \vert \mathbf{X}) \int_{Z} q (Z \vert \mathbf{X}) \log p(\mathbf{N}) dZ \\
        &= \sum_{\mathbf{N}} q (\mathbf{N} \vert \mathbf{X}) \log p(\mathbf{N})\\
        &= \sum_{c_1=1}^{C_1}\cdots \sum_{c_L=1}^{C_L} \gamma_{c_1,\dots,c_L} \log \mathbf{A}_{c_1,\dots,c_L}.
    \end{split}
    \end{align}
    \item We compute $\EE_{q(Z, \mathbf{N} \vert \mathbf{X})} \log q{(Z \vert \mathbf{X})} $ via the following:
    \begin{align}\begin{split}
        \EE_{q(Z, \mathbf{N} \vert \mathbf{X})} \log q{(Z \vert \mathbf{X})} 
        &=  \int_{Z} q(Z \vert \mathbf{X}) \log q(Z \vert \mathbf{X}) dZ =   -\frac{1}{2} \sum_{j=1}^{J} \left( \log (2\pi \sigma_{j}^2) + 1 \right)  .
    \end{split}
    \end{align}
    \item We compute $\EE_{q(Z, \mathbf{N} \vert \mathbf{X})} \log q{( \mathbf{N} \vert \mathbf{X})}$ via the following:
    \begin{align}\begin{split}
        \EE_{q(Z, \mathbf{N} \vert \mathbf{X})} \log q{( \mathbf{N} \vert \mathbf{X})} &= \sum_\mathbf{N} q(\mathbf{N} \vert \mathbf{X}) \int_{Z} q(Z \vert \mathbf{X}) \log q(\mathbf{N}\vert \mathbf{X}) dZ \\ 
        &= \sum_{\mathbf{N}} q(\mathbf{N} \vert \mathbf{X}) \log q(\mathbf{N} \vert \mathbf{X}) \\
        &= \sum_{c_1=1}^{C_1} \cdots \sum_{c_L=1}^{C_L} \gamma_{c_1,\dots,c_L} \log \gamma_{c_1,\dots,c_L}.
    \end{split}
    \end{align}
\end{enumerate}

Returning to the ELBO expression and combining all terms together, we have
\begin{equation} \begin{split}
\mathcal{L} &= \sum_{m=1}^M \EE_q \log p(X_m \vert Z,\mathbf{N}) +  \EE_q \log p(Z \vert \mathbf{N}) +  \EE_q \log p(\mathbf{N}) - \EE_q \log q(Z \vert \mathbf{X}) -  \EE_q \log q(\mathbf{N} \vert \mathbf{X}) \\
&= -\frac{1}{2}\sum_{m=1}^M \log (2 \pi \widehat{\sigma}_m^2) + \left \lVert \frac{X_m - \widehat{\mu}_m }{\widehat{\sigma}_m} \right \rVert^2 \\
&\qquad  - \frac{1}{2} \sum_{c_1=1}^{C_1}\cdots \sum_{c_L=1}^{C_L} \gamma_{c_1,\dots,c_L} \cdot \left[ \sum_{j=1}^{J} \log 2\pi \widetilde{\sigma}_{c_1,\dots,c_L;j}^2 + \frac{\sigma_{j}^2}{\widetilde{\sigma}_{c_1,\dots,c_L;j}^2} + \frac{(\mu_{j} - \widetilde{\mu}_{c_1,\dots,c_L;j})^2}{\widetilde{\sigma}_{c_1,\dots,c_L;j}^2} \right] \\
&\qquad + \sum_{c_1=1}^{C_1} \cdots \sum_{c_L=1}^{C_L} \gamma_{c_1,\dots,c_L} \log \mathbf{A}_{c_1,\dots,c_L} \\
&\qquad + \frac{1}{2} \sum_{j=1}^{J} \left( \log (2\pi \sigma_{j}^2) + 1 \right) \\
&\qquad - \sum_{c_1=1}^{C_1} \cdots \sum_{c_L=1}^{C_L} \gamma_{c_1,\dots,c_L} \log \gamma_{c_1,\dots,c_L} \\
\end{split} \end{equation}
Since any constant terms in $\mathcal{L}$ do not have bearing on the solution to the maximization problem, we can remove them; after rescaling, we have
\begin{equation} \begin{split}
\mathcal{L} &= - \sum_{m=1}^M \log (\widehat{\sigma}_m^2) + \left \lVert \frac{X_m - \widehat{\mu}_m}{\widehat{\sigma}_m}\right \rVert^2 \\
&\qquad + \sum_{j=1}^{J} \log(\sigma_{j}^2) \\
&\qquad + \sum_{c_1=1}^{C_1}\cdots \sum_{c_L=1}^{C_L} \gamma_{c_1,\dots,c_L} \cdot \left[ 2 \log\left(\frac{\mathbf{A}_{c_1,\dots,c_L}}{\gamma_{c_1,\dots,c_L}}\right) -  \sum_{j=1}^{J} \log (\widetilde{\sigma}_{c_1,\dots,c_L;j}^2) + \frac{\sigma_{j}^2}{\widetilde{\sigma}_{c_1,\dots,c_L;j}^2} + \frac{(\mu_{j} - \widetilde{\mu}_{c_1,\dots,c_L;j})^2}{\widetilde{\sigma}_{c_1,\dots,c_L;j}^2}  \right] \\
\end{split} \end{equation}
We describe how to compute $\mathbf{A}$ and $\gamma$ in Section \ref{sec:joint-def-A}.

In terms of architecture, the distributions are learned or computed in the following manner, where the $d$ subscripts index batching over several data points:
\begin{align} \begin{split}
    [\widehat{\mu}_m, \widehat{\sigma}_m] &= D_m(Z;\hat{\theta}_m), \text{ where } D_m \text{ is a neural network or expert model} \\
    [\mu_{m}, \sigma_{m}] &= F_{m}(X_m; \theta_m), \text{ where } F_{m} \text{ is a neural network} \\
\widetilde{\mu}_{c_1,\dots,c_L}  &= \frac{\sum_d  \mu^{(d)}  \gamma_{c_1, \dots, c_L}^{(d)}}{\sum_d  \gamma_{c_1, \dots, c_L}^{(d)}}, \text{ where } d \text{ indexes the } d^{th} \text{ data point and } \mu^{(d)} \text{ is the encoded mean of the } d^{th} \text{ data point}  \\
\widetilde{\sigma}_{c_1,\dots,c_L}^2 &=  \frac{\sum_d  (( \mu^{(d)} - \widetilde{\mu}_{c_1,\dots,c_L} )^2 + \sigma^{2(d)})\gamma_{c_1, \dots, c_L}^{(d)}}{\sum_d  \gamma_{c_1, \dots, c_L}^{(d)}}, \text{ where } d \text{ indexes the } d^{th} \text{data point } 
\end{split}
\end{align}

\section{Extension for general multivariate Gaussians}
\label{app:multivariate-gaussian-derivation}

Our ELBO is computationally tractable because our model uses Gaussians extensively. While we do restrict to Gaussians with diagonal covariances matrices, we show that computational tractability remains when using a generalized covariance matrix. We provide the lemma of \cite{jiang2016variational}, and then state and prove the generalized version.

\begin{corollary}  \cite{jiang2016variational}
Given Gaussian distributions $\mathbf{Y}_1 \sim \mathcal{N}(\mu_1, \sigma_1^2 \mathbf{I})$ and $\mathbf{Y}_2 \sim \mathcal{N}(\mu_2, \sigma_2^2 \mathbf{I})$ defined over the same probability space, where $\mu_1,\mu_2,\sigma_1^2,\sigma_2^2 \in \RR^J$, we have
\begin{equation}
\int_\Omega \Prob{\mathbf{Y}_1} \log \Prob{\mathbf{Y}_2}\, \text{d}\mu = - \frac{1}{2} \left( \sum_{j} \log(2\pi \sigma_{2,j}^2) + \frac{\sigma_{1,j}^2}{\sigma_{2,j}^2} + \frac{\left( \mu_{1,j} - \mu_{2,j} \right)^2}{\sigma_{2,j}^2} \right).
\label{eq:int-gaussian-log-gaussian-result}
\end{equation}
\label{lemma:int-gaussian-log-gaussian}
\end{corollary}

\begin{lemma}
Given Gaussian distributions $\mathbf{Y}_1 \sim \mathcal{N}(\vec{\mu}_1, \Sigma_1 )$ and $\mathbf{Y}_2 \sim \mathcal{N}(\vec{\mu}_2, \Sigma_2 )$ defined over the same probability space, where $\Sigma_1$ and $\Sigma_2$ are symmetric positive definite covariance matrices, and $\vec{\mu}_1,\vec{\mu}_2 \in \RR^J$ and $\Sigma_1,\,\Sigma_2\in\RR^{J\times J}$, we have
\begin{equation}
\int_{\RR^J} \Prob{\mathbf{Y}_1} \log \Prob{\mathbf{Y}_2} \text{d}\vec{y} = 
-\frac{J}{2}\log(2\pi) 
- \frac{1}{2}\log(\det(\Sigma_2)) 
- \frac{1}{2} (\vec{\mu}_1 - \vec{\mu_2})^T \Sigma_2^{-1} (\vec{\mu}_1 - \vec{\mu_2}) 
- \frac{1}{2}\trace(\Sigma_1\Sigma_2^{-1}).
\label{eq:int-multivariate-gaussian-log-gaussian-result}
\end{equation}
 \label{lemma:int-multivariate-gaussian-log-gaussian}
\end{lemma}
\begin{proof}
Recall the density function of a multivariate Gaussian random variable
\begin{equation}
\Prob{\mathbf{Y} = \vec{y}} = (2 \pi)^{-\frac{J}{2}} \det(\Sigma)^{-\frac{1}{2}} \exp\left(-\frac{1}{2}(\vec{y}-\vec{\mu})^T\Sigma^{-1}(\vec{y}-\vec{\mu})\right).
\end{equation}
By definition,
\begin{equation} \begin{split}
\int_{\RR^J} \Prob{\mathbf{Y}_1} \log \Prob{\mathbf{Y}_2} \text{d}\vec{y} &= \int_{\RR^J} (2 \pi)^{-\frac{J}{2}} \det(\Sigma_1)^{-\frac{1}{2}} \exp\left(-\frac{1}{2}(\vec{y}-\vec{\mu}_1)^T\Sigma_1^{-1}(\vec{y}-\vec{\mu}_1)\right) \\ 
&\qquad\qquad \cdot \log \left[ (2 \pi)^{-\frac{J}{2}} \det(\Sigma_2)^{-\frac{1}{2}} \exp\left(-\frac{1}{2}(\vec{y}-\vec{\mu}_2)^T\Sigma_2^{-1}(\vec{y}-\vec{\mu}_2)\right) \right] \text{d}\vec{y}\\
&= \int_{\RR^J} (2 \pi)^{-\frac{J}{2}} \det(\Sigma_1)^{-\frac{1}{2}} \exp\left(-\frac{1}{2}(\vec{y}-\vec{\mu_1})^T\Sigma_1^{-1}(\vec{y}-\vec{\mu_1})\right) \\ 
&\qquad\qquad \cdot \left( \frac{-J}{2} \log(2\pi) - \frac{1}{2}\log(\det(\Sigma_2)) -\frac{1}{2}(\vec{y}-\vec{\mu_2})^T\Sigma_2^{-1}(\vec{y}-\vec{\mu_2})\right)  \text{d}\vec{y}.\\
\end{split} \label{eq:int-multivariate-gaussian-1}\end{equation}
Since $\Sigma_1$ and $\Sigma_2$ are symmetric positive definite, their inverses are also symmetric positive definite and therefore have well-defined Cholesky factors, denoted by
\begin{equation}\begin{split}
\Sigma_1^{-1} &= L_1^T L_1 \\
\Sigma_2^{-1} &= L_2^T L_2.
\end{split}\end{equation}
Therefore, Equation \eqref{eq:int-multivariate-gaussian-1} becomes
\begin{equation} \begin{split}
\int_{\RR^J} \Prob{\mathbf{Y}_1} \log \Prob{\mathbf{Y}_2} \text{d}\vec{y} &=  -\frac{J}{2} \log(2\pi) - \frac{1}{2}\log(\det(\Sigma_2)) \\
&\qquad\qquad - \frac{1}{2} \int_{\RR^J} (2 \pi)^{-\frac{J}{2}} \det(\Sigma_1)^{-\frac{1}{2}} \exp\left(-\frac{1}{2}\lVert L_1(\vec{y}-\vec{\mu}_1)\rVert_2^2 \right) \lVert L_2(\vec{y}-\vec{\mu}_2) \rVert_2^2  \,  \text{d}\vec{y}.\\
\end{split} \label{eq:int-multivariate-gaussian-2} \end{equation}
We split
\begin{equation} \begin{split}
L_2(\vec{y}-\vec{\mu}_2) &= L_2 L_1^{-1} L_1 (\vec{y}-\vec{\mu}_1) + L_2(\vec{\mu}_1 - \vec{\mu}_2) \\
&:= \vec{x}_1 + \vec{x}_2
\end{split} \end{equation}
so that the remaining integral on the right-hand side of Equation \eqref{eq:int-multivariate-gaussian-2} becomes
\begin{equation} \begin{split}
\int_{\RR^J} (2 \pi)^{-\frac{J}{2}} & \det(\Sigma_1)^{-\frac{1}{2}} \exp\left(-\frac{1}{2}\lVert L_1(\vec{y}-\vec{\mu}_1)\rVert_2^2 \right) \lVert L_2(\vec{y}-\vec{\mu}_2) \rVert_2^2  \,  \text{d}\vec{y} \\
&=  \int_{\RR^J} (2 \pi)^{-\frac{J}{2}} \det(\Sigma_1)^{-\frac{1}{2}} \exp\left(-\frac{1}{2}\lVert L_1(\vec{y}-\vec{\mu}_1)\rVert_2^2 \right) \lVert \vec{x}_1 + \vec{x}_2 \rVert_2^2  \,  \text{d}\vec{y}\\
&= \int_{\RR^J} (2 \pi)^{-\frac{J}{2}} \det(\Sigma_1)^{-\frac{1}{2}} \exp\left(-\frac{1}{2}\lVert L_1(\vec{y}-\vec{\mu}_1)\rVert_2^2 \right) \left( \lVert \vec{x}_1 \rVert_2^2 + \lVert \vec{x}_2 \rVert_2^2 + 2 \langle \vec{x}_1, \vec{x}_2 \rangle \right) \, \text{d}\vec{y} \\
\end{split} \end{equation}
Since $\vec{x}_2$ is independent of $\vec{y}$, we can pull this term out of the integral. As the term involving $\langle \vec{x}_1, \vec{x}_2 \rangle$ is an odd function around $\vec{\mu_1}$, its integral equals 0. Therefore,
\begin{equation} \begin{split}
\int_{\RR^J} (2 \pi)^{-\frac{J}{2}} & \det(\Sigma_1)^{-\frac{1}{2}} \exp\left(-\frac{1}{2}\lVert L_1(\vec{y}-\vec{\mu}_1)\rVert_2^2 \right) \lVert L_2(\vec{y}-\vec{\mu}_2) \rVert_2^2  \,  \text{d}\vec{y} \\
&= \lVert L_2(\vec{\mu}_1 - \vec{\mu}_) \rVert_2^2 + \int_{\RR^J} (2 \pi)^{-\frac{J}{2}}  \det(\Sigma_1)^{-\frac{1}{2}} \exp\left(-\frac{1}{2}\lVert L_1(\vec{y}-\vec{\mu}_1)\rVert_2^2 \right) \lVert L_2 L_1^{-1} L_1 (\vec{y} - \vec{\mu}_1) \rVert_2^2 \, \text{d}\vec{y}.\\
\end{split} \end{equation}
With the change of variables $\vec{x} = L_1(\vec{y}-\vec{\mu}_1)$,
\begin{equation} \begin{split}
\int_{\RR^J} (2 \pi)^{-\frac{J}{2}} & \det(\Sigma_1)^{-\frac{1}{2}} \exp\left(-\frac{1}{2}\lVert L_1(\vec{y}-\vec{\mu}_1)\rVert_2^2 \right) \lVert L_2(\vec{y}-\vec{\mu}_2) \rVert_2^2  \,  \text{d}\vec{y} \\
&= \lVert L_2(\vec{\mu}_1 - \vec{\mu}_2) \rVert_2^2 + \int_{\RR^J} (2 \pi)^{-\frac{J}{2}}  \exp\left(-\frac{1}{2}\lVert \vec{x} \rVert_2^2 \right) \lVert L_2 L_1^{-1} \vec{x} \rVert_2^2 \, \text{d}\vec{x}\\
&= \lVert L_2(\vec{\mu}_1-\vec{\mu}_2)\rVert_2^2 + \trace(L_1^{-T} L_2^T L_2 L_1^{-1}) \\
&= \lVert L_2(\vec{\mu}_1-\vec{\mu}_2)\rVert_2^2 + \trace( \Sigma_1 \Sigma_2^{-1} ). \\
\end{split}\end{equation}
Returning to Equation \eqref{eq:int-multivariate-gaussian-2}, we finally have
\begin{equation} \begin{split}
\int_{\RR^J} \Prob{\mathbf{Y}_1} \log \Prob{\mathbf{Y}_2} \text{d}\vec{y} &=  \frac{-J}{2} \log(2\pi) - \frac{1}{2}\log(\det(\Sigma_2)) \\
&\qquad\qquad - \frac{1}{2} \int_{\RR^J} (2 \pi)^{-\frac{J}{2}} \det(\Sigma_1)^{-\frac{1}{2}} \exp\left(-\frac{1}{2}\lVert L_1(\vec{y}-\vec{\mu}_1)\rVert_2^2 \right) \lVert L_2(\vec{y}-\vec{\mu}_2) \rVert_2^2  \,  \text{d}\vec{y}\\
&= -\frac{J}{2} \log(2\pi) - \frac{1}{2}\log(\det(\Sigma_2)) \\
&\qquad\qquad - \frac{1}{2} \left( \lVert L_2 (\vec{\mu}_1 - \vec{\mu}_2) \rVert_2^2 + \trace(\Sigma_1 \Sigma_2^{-1})\right) \\
&= -\frac{J}{2} \log(2\pi) - \frac{1}{2}\log(\det(\Sigma_2)) - \frac{1}{2} (\vec{\mu_1}-\vec{\mu}_2)^T \Sigma_2^{-1} (\vec{\mu}_1 - \vec{\mu}_2) - \frac{1}{2}\trace(\Sigma_1 \Sigma_2^{-1}),
\end{split} \end{equation}
which concludes our proof.
\end{proof}
\begin{corollary}
When $\Sigma_1 = \vec{\sigma}_1 \mathbf{I}$ and $\Sigma_2 = \vec{\sigma}_2 \mathbf{I}$, Equation \eqref{eq:int-multivariate-gaussian-log-gaussian-result} simplifies to Equation \eqref{eq:int-gaussian-log-gaussian-result} in Lemma \ref{lemma:int-gaussian-log-gaussian}, i.e. the result in \cite{jiang2016variational}.
\end{corollary}

\end{document}